\documentclass[10pt]{article} 
\usepackage[english]{babel}
\usepackage{lmodern}
\usepackage[utf8]{inputenc}
\usepackage{balance}
\usepackage{xspace}
\usepackage{bbm}
\usepackage{rotating}
\usepackage{csquotes}
\usepackage{booktabs}
\usepackage{subfigure}
\usepackage{paralist}
\usepackage{multirow}
\usepackage{amsmath,amssymb,mathtools,amsthm}
\usepackage{graphicx}
\usepackage{xfrac}
\usepackage{array}
\usepackage{algorithm}
\usepackage[noend]{algpseudocode}
\usepackage{textcomp}
\usepackage{siunitx}
\usepackage[usenames,dvipsnames]{xcolor}
\usepackage{pgfplots}
\pgfplotsset{compat=newest,unit code/.code={\si{#1}},plot coordinates/math parser=false,grid style={lightgray}}
\usepgfplotslibrary{units,groupplots,fillbetween}
\usetikzlibrary{positioning,angles,quotes,patterns,shapes, backgrounds}
\tikzstyle{block} = [draw, rectangle, minimum height=2em, minimum width=5em]
\tikzstyle{addon} = [draw, rectangle, rounded corners]
\tikzstyle{pinstyle} = [pin edge={<-,thin,black}]
\tikzstyle{pinstyle2} = [pin edge={->,thin,black}]
\tikzstyle{mult} = [draw, isosceles triangle]
\tikzstyle{circ} = [draw, circle]
\tikzstyle{coord} = [coordinate]
\tikzstyle{circ2} = [draw, circle,minimum width=3pt, inner sep=0]
\tikzset{>=latex}
\tikzset{radiation/.style={{decorate,decoration={expanding
waves,angle=90,segment length=4pt}}}}
\usetikzlibrary{math}
\tikzmath
{
  function symlog(\x,\a){
    \yLarge = ((\x>\a) - (\x<-\a)) * (ln(max(abs(\x/\a),1)) + 1);
    \ySmall = (\x >= -\a) * (\x <= \a) * \x / \a ;
    return \yLarge + \ySmall ;
  };
  function symexp(\y,\a){
    \xLarge = ((\y>1) - (\y<-1)) * \a * exp(abs(\y) - 1) ;
    \xSmall = (\y>=-1) * (\y<=1) * \a * \y ;
    return \xLarge + \xSmall ;
  };
}

\graphicspath{{./figures/}}

\newtheorem{theo}{Theorem}
\newtheorem{lem}{Lemma}
\newtheorem{defi}{Definition}

\newtheorem{cor}{Corollary}
\newtheorem{remark}{Remark}
\newtheorem{assume}{Assumption}
\newtheorem{example}{Example}
\newcommand{\fakepar}[1]{\vspace{1mm}\noindent\textbf{#1.}}

\DeclareSIUnit{\belmilliwatt}{Bm}
\DeclareSIUnit{\dBm}{\deci\belmilliwatt}

\newcommand{\norm}[1]{\left\lVert#1\right\rVert}
\newcommand{\abs}[1]{\left\lvert#1\right\rvert}

\DeclareMathOperator*{\mbeq}{\overset{!}{=}}

\DeclareMathOperator*{\mmd}{\mathrm{MMD}}
\DeclareMathOperator*{\mmdsq}{\widehat{\mathrm{MMD}}^2}

\DeclareMathSymbol{:}{\mathord}{operators}{"3A}

\let\originalleft\left
\let\originalright\right
\renewcommand{\left}{\mathopen{}\mathclose\bgroup\originalleft}
\renewcommand{\right}{\aftergroup\egroup\originalright}

\newcommand{\eg}{e.g.,\xspace}
\newcommand{\ie}{i.e.,\xspace}

\newcommand{\cf}{cf.\@\xspace}
\newcommand{\capt}[1]{\mdseries{\emph{#1}}}

\newcommand{\iid}{i.i.d.\@\xspace}

\usepackage{ifthen}
\newboolean{authnotes}

\setboolean{authnotes}{false}

\ifthenelse{\boolean{authnotes}}
{
\newcommand{\db}[1]{\footnote{{\bf\color{green!50!black} Dominik: #1}}}
\newcommand{\st}[1]{\footnote{{\bf\color{purple!90!black} Sebastian: #1}}}
}
{
\newcommand{\db}[1]{}
\newcommand{\st}[1]{}
}

\usepackage[accepted]{tmlr}

\usepackage{amsmath,amsfonts,bm}

\def\figref#1{figure~\ref{#1}}

\def\secref#1{section~\ref{#1}}

\def\eqref#1{equation~\ref{#1}}
\def\Eqref#1{Equation~\ref{#1}}

\def\algref#1{algorithm~\ref{#1}}

\def\1{\bm{1}}

\DeclareMathAlphabet{\mathsfit}{\encodingdefault}{\sfdefault}{m}{sl}
\SetMathAlphabet{\mathsfit}{bold}{\encodingdefault}{\sfdefault}{bx}{n}

\newcommand{\E}{\mathbb{E}}

\newcommand{\R}{\mathbb{R}}

\newcommand{\Var}{\mathrm{Var}}

\usepackage{hyperref}
\usepackage{url}

\newcommand{\xI}{x^\mathrm{I}}
\newcommand{\xII}{x^\mathrm{II}}
\newcommand{\uI}{u^\mathrm{I}}
\newcommand{\uII}{u^\mathrm{II}}

\title{Identifying Causal Structure in Dynamical Systems}

\author{\name Dominik Baumann \email dominik.baumann@it.uu.se \\
        \addr Department of Information Technology\\
        Uppsala University\\
        Uppsala, Sweden
        \AND
       \name Friedrich Solowjow \email friedrich.solowjow@dsme.rwth-aachen.de \\
       \addr Institute for Data Science in Mechanical Engineering\\
       RWTH Aachen University\\
       Aachen, Germany
       \AND
       \name Karl H.\ Johansson \email kallej@kth.se \\
       \addr Division of Decision and Control Systems and Digital Futures\\
       KTH Royal Institute of Technology\\
       Stockholm, Sweden
       \AND
       \name Sebastian Trimpe \email trimpe@dsme.rwth-aachen.de \\
       \addr Institute for Data Science in Mechanical Engineering \\
       RWTH Aachen University\\
       Aachen, Germany}

\def\month{07}  
\def\year{2022} 
\def\openreview{\url{https://openreview.net/forum?id=X2BodlyLvT}} 

\begin{document}

\maketitle

\begin{abstract}
Mathematical models are fundamental building blocks in the design of dynamical control systems. 
As control systems are becoming increasingly complex and networked, approaches for obtaining such models based on first principles reach their limits.
Data-driven methods provide an alternative.
However, without structural knowledge, these methods are prone to finding spurious correlations in the training data, which can hamper generalization capabilities of the obtained models.
This can significantly lower control and prediction performance when the system is exposed to unknown situations.
A preceding causal identification can prevent this pitfall.
In this paper, we propose a method that identifies the causal structure of control systems.
We design experiments based on the concept of controllability, which provides a systematic way to compute input trajectories that steer the system to specific regions in its state space.
We then analyze the resulting data leveraging powerful techniques from causal inference and extend them to control systems.  
Further, we derive conditions that guarantee the discovery of the true causal structure of the system.  
Experiments on a robot arm demonstrate reliable causal identification from real-world data and enhanced generalization capabilities.
\end{abstract}

\section{Introduction}
\label{sec:intro}

When learning models for dynamical control systems, we ideally would like to obtain models that \emph{(i)} generalize well to new domains, \emph{(ii)} are interpretable, and \emph{(iii)} computationally efficient.
However, generalization to new domains and interpretability are particular weaknesses of current black-box machine learning methods \citep{scholkopf2022causality,rudin2019stop}.
We address this problem by first identifying a control system's causal structure and subsequently using this structural knowledge for model learning.

As the understanding of causality differs depending on the domain, we first provide some intuition of what kind of structure we seek to identify.
For this, following \cite{eichler2012causal}, we consider two notions.
The first is \emph{temporal precedence}: causes precede their effects. Temporal precedence is the typical causality notion used in systems theory \citep[p.~31]{hespanha2018linear}.
However, here, we mainly focus on the second notion, \emph{physical influence}: manipulating causes changes the effects.
In other words, we seek experimental routines and tests that enable control systems to learn \emph{(i)} what is the influence of their internal states on one another, and \emph{(ii)} which of their inputs influence which internal states.

Incorporating the causal structure into the model learning problem targets the three characteristics of ideal models outlined above.
A key reason for the shortcomings of most current machine learning algorithms concerning those characteristics is that they mostly rely on a pure statistical analysis of the data~\citep{pearl2018theoretical}.
Thus, when considering stochastic systems and finite sample sizes, these algorithms will, due to spurious correlations, typically find connections between variables, although those variables do not influence one another based on the underlying physical equations.
This can lead to catastrophic errors when extrapolating outside the training data, thus, diminishing generalization capabilities.
Further, such models are less interpretable since it is unclear whether connections between variables in the model indicate a causal influence or merely spurious correlation found in the training data.
A causal analysis mitigates both problems since it identifies which variables actually have a causal influence on each other.
Incorporating this knowledge into the model learning problem yields more interpretable models that generalize better to new domains \citep{pearl2018book}.
Lastly, when only connections between variables that actually have a causal influence on one another are considered, this also reduces the model's parameter space.
Therefore, the causal analysis also helps make the model computationally more efficient.

In this paper, we propose an algorithm that automatically identifies the causal structure of a control system.
The problem of inferring causal structure from observational data, \ie given data for which we cannot influence the data generating process,
 has been addressed using various methods, see \cite{spirtes2000causation,peters2017elements} for an overview.
For control systems, solely relying on observational data is not necessary.
Such systems are equipped with an input, which they can use to actively conduct experiments for causal inference.
Causal inference from experiments, or interventions, has also been studied, most prominently in the context of the do-calculus \citep{pearl1995causal}.
However, there it is assumed that state variables can be directly influenced by the input, which is often not possible in control systems.
In control systems, it is essential to consider a proper notion of \emph{controllability}, \ie how the system can be steered to particular regions in the state space through appropriate input trajectories.
To the best of our knowledge, such notions have not been considered yet in the causality literature.
In the literature on model learning and system identification, we find many techniques for learning mathematical models for control systems \citep{ljung1999system,nguyen2011model,schoukens2019nonlinear}.
In system identification, the problem of the exploding number of parameters for black-box methods has, for instance, been addressed using regularization techniques \citep{schoukens2019nonlinear}.
While these methods can reduce the number of parameters, they may exclude parameters representing a causal influence or include parameters representing spurious correlation depending on the regularization parameter.
Thus, the obtained models may fall short on generalization and interpretation capabilities.

\fakepar{Contributions}
In summary, existing methods from causal inference do not consider a proper notion of controllability, while existing methods from system identification cannot give formal guarantees on finding the true causal structure.
In this paper, we bridge this gap and present an algorithm that identifies a control system's causal structure through an experimental design based on a suitable controllability notion and a subsequent causal analysis, for which we can provide formal guarantees.
For the causal analysis, we leverage powerful kernel-based statistical tests based on the maximum mean discrepancy (MMD) \citep{gretton2012kernel}.
Since the MMD has been developed for independent and identically distributed (\iid) data, we extend it by deriving conditions under which the MMD still yields valid results and by coming up with a test statistic for hypothesis testing, despite non-\iid data.
In terms of controllability, we investigate three different settings: \emph{(i)} exact controllability, where we can exactly steer the system to a desired position, \emph{(ii)} stochastic controllability, where we can only steer the system to an $\epsilon$-region around the desired position, and \emph{(iii)} the special case of linear systems with Gaussian noise that are controllable in the sense of Kalman \citep{kalman1960contributions}, as they represent a widely studied class of systems.
We demonstrate the proposed method's applicability by automatically identifying the causal structure of a robotic system and a simulated, nonlinear quadruple tank system.
Further, we show improved generalization capabilities for the robotic system and reduced computational complexity for the quadruple tank system, both inherited through the causal identification.

\section{Related Work}
\label{sec:nips20_rel_work}

To the best of our knowledge, no other algorithm seeks to identify the causal structure of a dynamical control system through experiments based on a suitable controllability notion.
However, several works in causal inference aim to infer the causal structure of dynamical systems, and several methods in system identification seek to reduce the parameter space or identify structural properties of control systems.
In this section, we discuss those works.

\fakepar{Causal inference for dynamical systems}
Causal inference in dynamical systems or time series has been studied in \cite{demiralp2003searching,eichler2010graphical,moneta2011causal,malinsky2018causal} using vector autoregression, in \cite{peters2013causal} based on structural equation models, in \cite{entner2010causal}, using the fast causal inference algorithm \citep{spirtes2000causation}, and in \cite{salvi2021higher} and \cite{quinn2011estimating}, applying kernel mean embeddings and directed information, respectively.
In~\cite{pfister2019learning}, the authors develop a procedure for learning the causal structure of kinetic systems.
A more extensive overview of causal inference methods that can be applied to time-series data, with a particular focus on Earth system sciences, is provided in~\cite{runge2019inferring}.
While all these methods make causal inference for dynamical systems, none of them investigates experimental design.
Instead, they aim to infer the causal structure from observational data and, thus, need additional assumptions to arrive at statements about the causal structure.
Dynamical systems, as considered in this work, can be actively influenced through a control input.
Hence, we can design experiments and do not need to rely on the data being sufficiently rich.

\fakepar{Experimental design}
A well-known concept for causal inference from experiments is the do-calculus.
In the basic setting, a variable is clamped to a fixed value, and the distribution of the other variables conditioned on this intervention is studied \citep{pearl1995causal}.
Extensions to more general classes of interventions exist, see, \eg \cite{yang2018characterizing,shanmugam2015learning}, but they consider static models, which is different from the dynamical systems studied herein.
Causal inference in dynamical systems or time series with interventions has been investigated in \cite{eichler2012causal,peters2022causal,moij2013from,rubenstein2018from,sokol2014causal}.
However, therein it is assumed that one can directly manipulate the variables, \eg by setting them to fixed values or forcing them to follow a trajectory, which is typically impossible in practice.
None of those works considers various degrees of controllability, which is the case in this paper.
Thus, they are not readily applicable to control systems.

\fakepar{Model selection and regularization}
As an alternative to directly testing causal relations between variables, several methods exist that identify a dynamic model, trading off model complexity and accuracy.
Typically, this is done by letting the algorithm select from a set of candidate models.
In system identification, the Akaike information criterion \citep{akaike1998information} and the Bayesian information criterion \citep{schwarz1978estimating} are two well-known examples of such methods.
In neuroimaging, there are dynamic causal models \citep{friston2003dynamic,stephan2010ten}.
A third family of methods are symbolic regression techniques \citep{bongard2007automated,schmidt2009distilling,brunton2016discovering}.
In all cases, the true causal structure of the system can only be revealed if a model representing this structure is part of the candidate models.
Further, they typically use a regularization parameter to find a trade-off between model complexity and accuracy.
This parameter punishes model complexity (\eg the number of parameters) and rewards goodness of fit.
Thus, it also depends on the specific choice of this regularization parameter whether or not these algorithms find a model representing the system's true causal structure.

\fakepar{Structure detection in dynamical systems}
Revealing causal relations in a dynamical system can be interpreted as identifying its structure.
Related ideas exist in the identification of hybrid and piecewise affine systems \citep{roll2004identification,lauer2018hybrid}.
These approaches try to find a trade-off between model complexity and fit but cannot guarantee to find the true causal structure.
Further methods that identify structural properties of dynamical systems can be found in topology identification \citep{materassi2010topological,shahrampour2014topology,van2013identification} and complex dynamic networks \citep{boccaletti2006complex,liu2009structure,yu2010estimating}.
Those works seek to find interconnections between subsystems instead of identifying a system's inner structure as done herein.
Moreover, while the mentioned works often rely on restrictive assumptions such as known interconnections or linear dynamics,  
our approach can deal with nonlinear systems and does not require prior knowledge.

\fakepar{Kernel mean embeddings}
For causal inference, we will leverage concepts based on kernel mean embeddings, which have been widely used for causal inference \citep{peters2017elements,chen2014causal,lopez2015towards}.
A downside of those methods is that they typically assume that data has been drawn \iid from the underlying probability distributions.
Extensions to non-\iid data exist \citep{chwialkowski2014kernel,chwialkowski2014wild}, but rely on mixing time arguments.
Dynamical systems, as investigated in this work, often have large mixing times or do not mix at all \citep{simchowitz2018learning}. 
Therefore, these types of analyses are not sufficient in this case.

\section{Problem Setting and Main Idea}
\label{sec:problem}
We consider dynamical control systems of the form
\begin{align}
\label{eqn:gen_sys_nl}
 x(t) = f(x(0),u(0\,:\,t),v(0\,:\,t))
\end{align}
with discrete time index $t\in\mathbb{N}$, dynamics function $f$, state $x(t)\in\mathcal{X}\subset\R^n$, state space $\mathcal{X}$, input $u(t)\in\mathcal{U}\subset\R^m$, input space $\mathcal{U}$, and $v(t)\in\R^n$ an independent random variable sequence.
The notation $(0\,:\,t)$ here denotes the sequence $u(0), u(1), \ldots, u(t)$.
In the following, we will omit this notation and simply write $u$ or $v$ if we consider the whole trajectory. 
The description of the system in \eqref{eqn:gen_sys_nl} is different from the standard, incremental version $x(t+1) = \tilde{f}(x(t),u(t),v(t))$.
Specifically, \eqref{eqn:gen_sys_nl} emphasizes the dependence of $x$ at time $t$ on the initial state.
\Eqref{eqn:gen_sys_nl} can readily be obtained by iterating $\tilde{f}$ starting from $x(0)$.

We will lend notation from the do-calculus \citep{pearl1995causal} for causal inference.
In this notation, $\mathbb{P}(x(t)\mid \mathrm{do}(x_i(0)=\xI_i(0)))$ defines the conditional probability distribution of $x(t)$ given that we \emph{set} the initial condition of the $i$th component of the state vector $x$ to a specific value $\xI_i(0)$ while all other components of $x$ and the inputs $u$ are left unaffected.
In this article, we show how such do-conditional distributions can be realized in dynamical systems when considering a suitable notion of controllability and how we can use them for causal inference.

\subsection{Problem Setting}
The system description from \eqref{eqn:gen_sys_nl} is stochastic and, hence, induces a probability distribution $\mathbb{P}(x)$ over trajectories $x$.
Based on this, we define non-causality adapting standard notation from the do-calculus \citep{pearl1995causal} to dynamical systems as in \eqref{eqn:gen_sys_nl}:
\begin{defi}[Global non-causality]
\label{def:independence}
The state variable $x_j$ \emph{does not cause} $x_i$ if $\mathbb{P}(x_i\mid \mathrm{do}(x_j(0)=\xI_j(0)))=\mathbb{P}(x_i\mid \mathrm{do}(x_j(0)=\xII_j(0)))$ for all $\xI_j(0)$ and $\xII_j(0)$.
The superscripts \emph{I} and \emph{II} denote two experimental designs, \ie different choices of $x_j(0)$.
Similarly, $u_j$ does not cause $x_i$ if $\mathbb{P}(x_i\mid \mathrm{do}(u_j=\uI_j)) = \mathbb{P}(x_i\mid \mathrm{do}(u_j=\uII_j))$ for all $\uI_j$ and $\uII_j$.
\end{defi} 
The main objective of this article is to develop an algorithm to test for non-causality in the sense of definition~\ref{def:independence} with theoretical guarantees.
That is, we seek to test whether the initial condition of $x_j$ respectively the input trajectory $u_j$ changes the probability distribution of the resulting $x_i$ trajectory.
The \emph{do}-operator denotes that we force $x_j$ to a specific initial condition and $u_j$ to be a specific input trajectory, while all other initial conditions and input trajectories remain fixed between experiments.
While this is a common assumption of the do-calculus, in dynamical systems, we cannot just \emph{set} state variables to specific values.
Instead, we need to consider a proper notion of \emph{controllability}.
In the following, we discuss how the do-operator can be applied to dynamical systems, \ie how we can, based on an appropriate controllability notion, \emph{steer} state variables toward specific values as required to test for non-causality.
Subsequently, we develop an algorithm that identifies for each pair $x_i$ and $x_j$ and for each pair $x_i$ and $u_j$ of a control system as in \eqref{eqn:gen_sys_nl}, whether or not both variables are non-causal according to definition~\ref{def:independence}.

\subsection{Main Idea}
For causal inference, we will exploit that we can influence \eqref{eqn:gen_sys_nl} through $u$.
To make the notion of \emph{how} we can influence the system precise, we adopt controllability definitions for stochastic systems \citep{sunahara1974stochastic,bashirov1997controllability}:
\begin{defi}
\label{def:eps_controllability}
The system described by \eqref{eqn:gen_sys_nl} is said to be \emph{completely $\epsilon$-controllable} in probability $\eta$ in the normed square sense in the time interval $[0,t_\mathrm{f}]$ if for all desired states $x_\mathrm{des}$ and initial states $x(0)$ from $\mathcal{X}$, there exists an input sequence $u$ from $\mathcal{U}$ such that $\mathrm{Pr}\{\norm{x(t_\mathrm{f})-x_\mathrm{des}}_2^2\ge\epsilon\}\le 1-\eta$, with $0<\eta<1$.
\end{defi}
A variety of methods exist to identify or learn models for systems as in \eqref{eqn:gen_sys_nl}, \eg Gaussian process regression \citep{williams2006gaussian} or fitting linear state space models using least squares \citep{ljung1999system}.
In the following, we assume that we can obtain an estimate $\hat{f}$ of the investigated system in \eqref{eqn:gen_sys_nl} (including an estimate of the distribution of $v(t)$) using techniques from system identification and model learning.
The concrete method that is used to obtain $\hat{f}$ is independent of the developed causal identification procedure.
However, in later sections, we will specify requirements the model estimate $\hat{f}$ needs to satisfy.
This model estimate $\hat{f}$, obtained without further assumptions or physical insights, will, due to the stochasticity of \eqref{eqn:gen_sys_nl}, almost surely entail spurious correlation and suggest causal influences that are actually not present in the physical system.
Nevertheless, it will allow us to (approximately) steer the system to specific initial conditions and start experiments from there.

We propose two types of experiments to test for causal relations.
For the first type, we investigate whether $x_j$ causes $x_i$.
We conduct two experiments (denoted by I and II) with different initial conditions $\xI_j(0)\neq\xII_j(0)$, while all others are kept the same (\cf \figref{fig:exp_design}).
This can be formalized as
\begin{subequations}
\label{eqn:exp_design}
\begin{align}
\label{eqn:exp_design_state}
\xI_\ell(0) = \xII_\ell(0) \text{ for all } \ell \neq j,\quad 
\xI_j(0) \neq \xII_j(0), \quad
\uI_\ell(t) = \uII_\ell(t) \text{ for all } \ell,t.
\end{align}
By comparing the resulting trajectories of $\xI_i$ and $\xII_i$, we can check whether the change in $x_j(0)$ caused a different behavior.
For checking the similarity of trajectories, we will use the MMD, whose mathematical definition we provide in the next section.
\begin{figure}
\centering
\begin{tikzpicture}
\node[block](exp1){Exp I};
\node[block, below = 4em of exp1](exp2){Exp II};
\node[left = 5em of exp1](init1){$\xI_j(0)$};
\node[left = 5em of exp2](init2){$\xII_j(0)$};
\node[align = center] at($(init1)!0.5!(init2)$)(rest){$\xI_{1\ldots n\setminus j}(0)=\xII_{1\ldots n\setminus j}(0)$,\\$\uI=\uII$};
\node[right = 3em of exp1](res1){$\xI_i$};
\node at(res1|-exp2)(res2){$\xII_i$};

\draw[->](init1) -- (exp1);
\draw[->](init2) -- (exp2); 
\draw[->](rest) -- (exp1);
\draw[->](rest) -- (exp2);
\draw[->](exp1) -- (res1);
\draw[->](exp2) -- (res2);
\draw[<->,thick](res1)--node[midway,right]{MMD}(res2);
\end{tikzpicture}
\caption{Experimental design for causal inference.
\capt{We design two experiments, where all initial conditions and input trajectories are constant except for $x_j(0)$.
If the resulting trajectories of $x_i$ are different in both experiments, we have evidence that the change in $x_j(0)$ caused this change.}}
\label{fig:exp_design}
\end{figure}
The second type of experiments is analogous to the first, but instead of varying initial conditions, we consider different input trajectories $\uI_j\neq \uII_j$,
\begin{align}
\label{eqn:exp_design_inp}
\xI_\ell(0) = \xII_\ell(0) \text{ for all } \ell, \quad
\uI_\ell(t) = \uII_\ell(t) \text{ for all } \ell\neq j, \text{ for all } t,\quad
\uI_j(t) \neq \uII_j(t)\text{ for all } t.
\end{align}
\end{subequations}
\begin{remark}
Note that for the testing experiments, we design open-loop trajectories.
That is, the input cannot depend on the system's current state.
This is essential since it creates independent trajectories for which we can leverage the MMD as a similarity measure.
\end{remark}

\subsection{Road-map}
We propose an algorithm that consists of two steps: \emph{(i)} we design experiments, and \emph{(ii)} we analyze the resulting data using the MMD.
The data obtained from these experiments must fulfill specific requirements to allow for proper causal inference.
In the next section, we provide the mathematical definition of the MMD and deduce the requirements on the experimental data.
We then subsequently develop a suitable experimental design based on the controllability notion stated in definition~\ref{def:eps_controllability}.
We state conditions under which this experimental design yields data from which the MMD can provably infer the true causal structure in the infinite sample limit and derive a hypothesis test for finitely many samples.
Until here, we focus on a rigorous convergence analysis.
However, this requires us to do many experiments on the real system.
In \secref{sec:implementation}, we propose a heuristic algorithm that is more efficient in terms of the number of experiments and computations but forgoes some of the guarantees.
Lastly, in \secref{sec:eval}, we demonstrate the method's applicability on a real robotic system and a quadruple tank process and present comparisons with a sparse identification and a causal discovery method on a synthetic linear example.

\section{Causal Identification for Dynamical Systems}
\label{sec:tec_res}

We will now develop the causality testing procedure.
First, we introduce the MMD \citep{gretton2012kernel}, which we shall use as a similarity measure.
The MMD can be used to check whether two probability distributions $\mathbb{P}$ and $\mathbb{Q}$ are equal based on samples drawn from these distributions.
Let $X$ and $Y$ be samples drawn \iid from $\mathbb{P}$ and $\mathbb{Q}$, respectively.
Further, let $\mathcal{H}$ be a \emph{reproducing kernel Hilbert space} \citep{sriperumbudur2010hilbert}, with canonical feature map $\phi:\, \mathcal{X}\to\mathcal{H}$.
The MMD is defined as
\begin{align}
\label{eqn:mmd_gen}
\mmd (\mathbb{P},\mathbb{Q}) = \lVert \mathbb{E}_{X\sim \mathbb{P}}[\phi(X)] - \mathbb{E}_{Y\sim \mathbb{Q}}[\phi(Y)] \rVert_\mathcal{H}.
\end{align}
The feature map $\phi$ can be expressed in terms of a kernel function $k(\cdot,\cdot)$, where $k(x,y)=\langle\phi(x),\phi(y)\rangle_\mathcal{H}$.
If the kernel is characteristic, we have $\mmd (\mathbb{P},\mathbb{Q}) = 0$ if, and only if, $\mathbb{P}=\mathbb{Q}$ \citep{fukumizu2008kernel,sriperumbudur2011universality}.
In the remainder of the paper, we always assume a characteristic kernel (\eg the Gaussian kernel).

\begin{remark}
In general, also other measures that compare probability distributions may be used for our algorithm.
An overview of such methods is provided in \cite{sriperumbudur2012empirical}.
A popular example would be the Kullback-Leibler divergence \citep{kullback1951information}.
In this paper, we propose to use the MMD since it allows us to compare probability distributions without actually estimating them, provides theoretical guarantees, and can be computed efficiently.
\end{remark}

In the following, we derive conditions that allow one to provably identify causal relations.
We investigate three settings.
First, we discuss the case where we can precisely steer the system to desired initial conditions (\ie $\epsilon =0$ in definition~\ref{def:eps_controllability}).
We then extend this to $\epsilon > 0$, which requires a stricter controllability definition.
Finally, we show that for linear systems with additive Gaussian noise, a widely studied class of systems, the conditions stated by Kalman \citep{kalman1960contributions} are sufficient, and the identification is substantially easier.

\subsection{Exact Controllability}
\label{sec:exact_ctrl}
When considering control systems, instead of obtaining single \iid samples from stationary distributions, we receive sequences of random variables sampled from a stochastic process as in \eqref{eqn:gen_sys_nl}.
This data is often non-\iid.
The objects of interest, whose distributions we want to compare, are then the $x_i$ trajectories obtained from two different experimental settings.
To simplify notation, we denote a trajectory obtained from the first setting by $\xI_i$ and the joint probability distribution over the trajectory states by $\mathbb{P}^\mathrm{I}$, and equivalently for $\xII_i$ and $\mathbb{P}^\mathrm{II}$.
We sample from $\mathbb{P}^\mathrm{I}$ and $\mathbb{P}^\mathrm{II}$ by designing two experiments as in \eqref{eqn:exp_design} with fixed length $T$ and repeating each experiment $m$ times.
That is, we obtain $2m$ sequences of $T$ random variables sampled at discrete intervals of fixed length (\ie $t\to t+1$ always has the same length).
These samples we denote by $\textbf{x}^\mathrm{I}_i,\textbf{x}^\mathrm{II}_i\in\R^{m\times T}$.
Note that \emph{(i)} all $m$ sequences are sampled from the same distributions $\mathbb{P}^\mathrm{I}$ and $\mathbb{P}^\mathrm{II}$, \emph{(ii)} while the distributions are non-stationary along time, the distributions for multiple experiments of fixed length $T$ are identical, and \emph{(iii)} all $m$ sequences are independent of each other.
The MMD then reads
\begin{align}
\label{eqn:mmd_dyn_sys}
\begin{split}
\mmd (\textbf{x}_i^\mathrm{I},\textbf{x}_i^\mathrm{II}) = \lVert \mathbb{E}_{\textbf{x}_i^\mathrm{I}\sim \mathbb{P}^\mathrm{I}}[\phi(\textbf{x}_i^\mathrm{I})]
- \mathbb{E}_{\textbf{x}_i^\mathrm{II}\sim \mathbb{P}^\mathrm{II}}[\phi(\textbf{x}_i^\mathrm{II})] \rVert_\mathcal{H}.
\end{split}
\end{align}
If we now design experiments using \eqref{eqn:exp_design}, we can check the similarity of $x_i$ trajectories with \eqref{eqn:mmd_dyn_sys}.
Following \cite{szabo2018characteristic}, if we use a characteristic kernel in \eqref{eqn:mmd_dyn_sys}, the general properties of the MMD test are still valid given that the initial conditions $\xI(0)$ and $\xII(0)$, respectively the input trajectories $\uI$ and $\uII$, are \iid, which is the case for our design.
That is, $\mmd>0$ suggests that the distributions of the trajectories are different, so we can conclude that there is a causal influence.
However, the other direction is less straightforward: for a system as in \eqref{eqn:gen_sys_nl}, the MMD may be zero, even though variables are dependent, as can be seen in the following example:
\begin{example}
\label{exp:loc_ind}
Assume a control system with $x_1(t+1) = x_1(t)x_2(t)$ and $x_2(t+1) = u(t)$, and an input signal $u(t)$ that is different from $0$.
If we choose $x_1(0)=0$, the trajectory of $x_1$ will, despite the fact that $x_2$ clearly has a causal influence on $x_1$, always be $0$ no matter the initial condition $x_2(0)$.
\end{example}
To address this, we define the concept of local non-causality:
\begin{defi}[Local non-causality]
\label{def:loc_ind}
Let $\mathcal{X}_\mathrm{nc}\subset\mathcal{X}$ and $\mathcal{U}_\mathrm{nc}\subset\mathcal{U}$ with $\mathcal{X}_\mathrm{nc}\cup\mathcal{U}_\mathrm{nc}\neq\varnothing$.
The state variable $x_j$ does \emph{locally not cause} $x_i$ if $\mathbb{P}(x_i\mid \mathrm{do}(x_j(0)=\xI_j(0)))=\mathbb{P}(x_i\mid \mathrm{do}(x_j(0)=\xII_j(0)))$ for all $\xI_j(0)$ and $\xII_j(0)$ given that the sequence $x$ is entirely in $\mathcal{X}_\mathrm{nc}$ and the sequence $u$ is entirely in $\mathcal{U}_\mathrm{nc}$.\footnote{In general, there may exist different $\mathcal{X}^{ij}_\mathrm{nc}$ for each combination of $x_i$ and $x_j$ (and likewise for $\mathcal{U}_\mathrm{nc}$).
We can also cover this case. However, we omit it here to simplify notation.}
Similarly, the input variable $u_j$ does locally not cause $x_i$ if $\mathbb{P}(x_i\mid \mathrm{do}(u_j=\uI_j))=\mathbb{P}(x_i\mid \mathrm{do}(u_j=\uII_j))$ for all $\uI_j$ and $\uII_j$ given that $x$ is in $\mathcal{X}_\mathrm{nc}$ and $u$ in $\mathcal{U}_\mathrm{nc}$.
\end{defi}
The non-causality becomes global if $\mathcal{X}_\mathrm{nc} = \mathcal{X}$ and $\mathcal{U}_\mathrm{nc} = \mathcal{U}$.
To properly test for causal relations, we need to ensure that the experimental design in \eqref{eqn:exp_design} yields initial conditions and input trajectories that are not inside $\mathcal{X}_\mathrm{nc}$ and $\mathcal{U}_\mathrm{nc}$.
For this, we propose to design experiments based on the estimated model $\hat{f}$.
In particular, we utilize simulated trajectories based on the model $\hat{f}$, which we denote by $(\hat{x}\mid \hat{f})$ and where for each $t$ we have $\hat{x}(t) = \hat{f}(\hat{x}(0), u, \hat{v})$.
We then need to make the following assumption about the system described by \eqref{eqn:gen_sys_nl} and the estimated model $\hat{f}$:
\begin{assume}
\label{assume:no_loc_ind}
Consider a dynamical system as in \eqref{eqn:gen_sys_nl} for which $\mathcal{X}_\mathrm{nc}\cup\mathcal{U}_\mathrm{nc}\neq\varnothing$.
Consider further two independent experimental designs following \eqref{eqn:exp_design} with initial conditions $\xI(0), \xII(0)$ and input trajectories $\uI, \uII$ for the first, and initial conditions $x^\mathrm{III}(0), x^\mathrm{IV}(0)$ and input trajectories $u^\mathrm{III},u^\mathrm{III}$ for the second experiment.
Assume that all individual inputs $u(t)$ that are part of $u^\mathrm{III}$ or $u^\mathrm{IV}$ are in $\mathcal{U}_\mathrm{nc}$ and that all states $x(t)$ that are part of the simulated trajectories $(\hat{x}^\mathrm{III}\mid\hat{f})$ or $(\hat{x}^\mathrm{IV}\mid\hat{f})$ are inside $\mathcal{X}_\mathrm{nc}$.
Further, assume that for the first experiment there exists some $u(t)$ or some $x(t)$ that is not part of $\mathcal{U}_\mathrm{nc}$ respectively $\mathcal{X}_\mathrm{nc}$.
For such cases, we assume $\mmd(\hat{x}_i^\mathrm{III},\hat{x}_i^\mathrm{IV}\mid\hat{f})<\mmd(\hat{x}_i^\mathrm{I},\hat{x}_i^\mathrm{II}\mid\hat{f})$.
\end{assume}
This assumption \emph{does not} require that $\hat{f}$ captures the causal structure.
That is, if variable $x_j$ does not cause variable $x_i$, the model $\hat{f}$ may still suggest that they are causally related---such modeling errors will then be accounted for through the identification procedure we propose in this paper.
However, we require that $\hat{f}$ captures that the influence of $x_j$ on $x_i$ is lower in regions in which they are locally non-causal.
Intuitively, the assumption says that if $x_i$ does not influence $x_j$ in certain parts of the state/action space, our model may, due to spurious correlation, still assign some influence of $x_i$ on $x_j$ in those regions. 
Nevertheless, we expect the model to assign a stronger influence in regions where the influence stems not only from spurious correlations but also from an actual physical influence of $x_i$ on $x_j$.
Given a reasonable choice of system identification or model learning technique and a sufficiently rich excitation signal, this assumption will typically be satisfied in practice, as shown in the example below.
It further follows that simulated trajectories in regions of local non-causality have a smaller MMD than trajectories in other regions.
\addtocounter{example}{-1}
\begin{example}[cont]
Given assumption~\ref{assume:no_loc_ind}, if we simulate experiments using \eqref{eqn:exp_design_state} and compute the MMD for the resulting $\hat{x}_1$, the MMD will be \emph{lower} if we choose $\hat{x}_1(0)=0$ than for any other choice of $\hat{x}_1(0)$.
To validate assumption~\ref{assume:no_loc_ind}, we identify the system using a regularization based nonlinear system identification algorithm called sparse identification of nonlinear dynamics (SINDy) \citep{brunton2016discovering}.
By exciting the system for \num{500} time steps and using the SINDy implementation from \cite{desilva2020pysindy}, we receive a model.
Simulating the MMD for different initial conditions $x_1(0)$ for $x_2^\mathrm{I}(0)=-10$ and $x_2^\mathrm{II}(0)=10$ then reveals that the MMD has a minimum at $x_1(0)=0$ (\cf \figref{fig:ex1_mmd}), \ie the identified model satisfies assumption~\ref{assume:no_loc_ind}.
We will introduce the finite sample approximation for the MMD used in \figref{fig:ex1_mmd} in \secref{sec:finite_samples} and we provide a further example to empirically verify assumption~\ref{assume:no_loc_ind} for a system with a hysteresis in appendix~\ref{sec:example_loc_non_caus}.
\end{example}
\begin{figure}
\centering
\begin{tikzpicture}

\definecolor{color0}{rgb}{0.12156862745098,0.466666666666667,0.705882352941177}

\begin{axis}[
ylabel={$\mmdsq$},
xlabel={$x_1(0)$},
label style={font=\scriptsize},
tick label style={font=\scriptsize},
axis lines=left,
grid=both,
width=0.7\textwidth,
height=0.2\textheight,
ymax=0.0035,
]
\addplot [semithick, color0]
table {%
-10 0.00269285258202485
-9.8 0.00269285258202485
-9.6 0.00269285258202485
-9.4 0.00269285258202485
-9.2 0.00269285258202485
-9 0.00269285258202485
-8.8 0.00269285258202485
-8.6 0.00269285258202485
-8.40000000000001 0.00269285258202485
-8.20000000000001 0.00269285258202485
-8.00000000000001 0.00269285258202485
-7.80000000000001 0.00269285258202485
-7.60000000000001 0.00269285258202485
-7.40000000000001 0.00269285258202485
-7.20000000000001 0.00269285258202485
-7.00000000000001 0.00269285258202485
-6.80000000000001 0.00269285258202485
-6.60000000000001 0.00269285258202485
-6.40000000000001 0.00269285258202485
-6.20000000000001 0.00269285258202485
-6.00000000000001 0.00269285258202485
-5.80000000000001 0.00269285258202485
-5.60000000000002 0.00269285258202485
-5.40000000000002 0.00269285258202485
-5.20000000000002 0.00269285258202485
-5.00000000000002 0.00269285258202485
-4.80000000000002 0.00269285258202485
-4.60000000000002 0.00269285258202485
-4.40000000000002 0.00269285258202485
-4.20000000000002 0.00269285258202485
-4.00000000000002 0.00269285258202485
-3.80000000000002 0.00269285258202485
-3.60000000000002 0.00269285258202485
-3.40000000000002 0.00269285258202485
-3.20000000000002 0.00269285258202485
-3.00000000000002 0.00269285258202485
-2.80000000000003 0.00269285258202485
-2.60000000000003 0.00269285258202485
-2.40000000000003 0.00269285258202485
-2.20000000000003 0.00269285258202485
-2.00000000000003 0.00269285258202485
-1.80000000000003 0.00269285258202485
-1.60000000000003 0.00269285258202485
-1.40000000000003 0.00269285258202485
-1.20000000000003 0.00269285258202485
-1.00000000000003 0.00269285258202485
-0.800000000000033 0.00269285258202485
-0.600000000000033 0.00269285258202485
-0.400000000000034 0.00269285258202485
-0.200000000000035 0.00269285258202485
0 0
0.199999999999964 0.00269285258202485
0.399999999999963 0.00269285258202485
0.599999999999962 0.00269285258202485
0.799999999999962 0.00269285258202485
0.999999999999961 0.00269285258202485
1.19999999999996 0.00269285258202485
1.39999999999996 0.00269285258202485
1.59999999999996 0.00269285258202485
1.79999999999996 0.00269285258202485
1.99999999999996 0.00269285258202485
2.19999999999996 0.00269285258202485
2.39999999999996 0.00269285258202485
2.59999999999996 0.00269285258202485
2.79999999999995 0.00269285258202485
2.99999999999995 0.00269285258202485
3.19999999999995 0.00269285258202485
3.39999999999995 0.00269285258202485
3.59999999999995 0.00269285258202485
3.79999999999995 0.00269285258202485
3.99999999999995 0.00269285258202485
4.19999999999995 0.00269285258202485
4.39999999999995 0.00269285258202485
4.59999999999995 0.00269285258202485
4.79999999999995 0.00269285258202485
4.99999999999995 0.00269285258202485
5.19999999999995 0.00269285258202485
5.39999999999995 0.00269285258202485
5.59999999999994 0.00269285258202485
5.79999999999994 0.00269285258202485
5.99999999999994 0.00269285258202485
6.19999999999994 0.00269285258202485
6.39999999999994 0.00269285258202485
6.59999999999994 0.00269285258202485
6.79999999999994 0.00269285258202485
6.99999999999994 0.00269285258202485
7.19999999999994 0.00269285258202485
7.39999999999994 0.00269285258202485
7.59999999999994 0.00269285258202485
7.79999999999994 0.00269285258202485
7.99999999999994 0.00269285258202485
8.19999999999994 0.00269285258202485
8.39999999999993 0.00269285258202485
8.59999999999993 0.00269285258202485
8.79999999999993 0.00269285258202485
8.99999999999993 0.00269285258202485
9.19999999999993 0.00269285258202485
9.39999999999993 0.00269285258202485
9.59999999999993 0.00269285258202485
9.79999999999993 0.00269285258202485
};
\end{axis}

\end{tikzpicture}
\caption{MMD of simulated experiments with different initial conditions $x_1(0)$ for the system from example~\ref{exp:loc_ind}. \capt{The MMD has a minimum at $x_1(0)=0$, \ie it reflects the local non-causality. The finite sample approximation $\mmdsq$ of the MMD used in this figure will be introduced in \secref{sec:finite_samples}.}}
\label{fig:ex1_mmd}
\end{figure}
We now specify the experimental design.
To avoid regions of local non-causality, 
we propose to maximize the MMD given the model estimate.
Thus, for checking whether $x_j$ causes $x_i$, we choose input trajectories of length $T$ and initial conditions that solve
\begin{subequations}
\label{eqn:max_mmd}
\begin{align}
\label{eqn:max_mmd_state}
\begin{split}
&\max_{\xI(0),\xII(0),\uI,\uII} \mmd(\hat{x}^\mathrm{I}_i,\hat{x}^\mathrm{II}_i\mid\hat{f})\\
&\text{subject to } \xI_\ell(0) = \xII_\ell(0)\text{ for all } \ell \neq j 
 \quad \uI_\ell(t) = \uII_\ell(t)\text{ for all } \ell,t .
 \end{split}
\end{align}
We will discuss how to handle this optimization problem in practice in \secref{sec:framework}.
In particular, in practice, we typically do not require a global solution to \eqref{eqn:max_mmd_state}.
If we want to check whether $u_j$ causes $x_i$, we choose input trajectories and initial conditions by solving
\begin{align}
\label{eqn:max_mmd_input}
\begin{split}
&\max_{\xI(0),\xII(0),\uI,\uII} \mmd(\hat{x}^\mathrm{I}_i,\hat{x}^\mathrm{II}_i\mid\hat{f})\\
&\text{subject to }\xI_\ell(0) = \xII_\ell(0)\text{ for all } \ell \quad
 \uI_\ell(t) = \uII_\ell(t)\text{ for all } \ell\neq j,t .
 \end{split}
\end{align}
\end{subequations}
Before stating our main theorem, we need one further assumption.
Given the system model from \eqref{eqn:gen_sys_nl}, we could also have systems for which the influence of $x_j$ on $x_i$ becomes only apparent after a certain number of time steps. 
To guarantee that we identify the true causal structure, we need to assume that this delay is at most equal to the length of an experiment $T$.
\begin{assume}
\label{ass:delay}
Consider a pair $(x_i, x_j)$ for which $x_j$ has a causal influence on $x_i$ as per definition~\ref{def:loc_ind}.
We assume that changes of $x_j$, when outside of regions of local non-causality, cause a change in $x_i$ after at most $T$ discrete time steps.
\end{assume}
We can now state the main theorem:
\begin{theo}
\label{thm:general_result}
Consider a completely $\epsilon$-controllable system described by \eqref{eqn:gen_sys_nl} with $\epsilon=0$ that fulfills assumptions~\ref{assume:no_loc_ind} and~\ref{ass:delay}.
Let experiments be designed according to \eqref{eqn:max_mmd} for a fixed experiment length $T$ and repeated infinitely often (\ie we have $\textbf{x}_i^\mathrm{I}, \textbf{x}_i^\mathrm{II}\in\R^{m\times T}$ with $m\to\infty$). Then:  $\mmd(\textbf{x}^\mathrm{I}_i,\textbf{x}^\mathrm{II}_i)=0$ if, and only if, $x_j$ respectively $u_j$ does not cause $x_i$ as per definition~\ref{def:independence}.
\end{theo}
\begin{proof}
Let variables be non-causal.
Then, we have by definition~\ref{def:independence}, $\mathbb{P}(x_i\mid \mathrm{do}(x_j(0)=\xI_j(0)))=\mathbb{P}(x_i\mid \mathrm{do}(x_j(0)=\xII_j(0)))$.
That is, the distribution of $x_i$ in both experiments is equal. 
Thus, $\mmd(\textbf{x}^\mathrm{I}_i,\textbf{x}^\mathrm{II}_i)=0$ follows from \cite{gretton2012kernel}.
Now, assume $\mmd(\textbf{x}^\mathrm{I}_i,\textbf{x}^\mathrm{II}_i)=0$.
This implies that the distribution of $x_i$ is equal in both experiments \citep{gretton2012kernel}, \ie $\mathbb{P}(x_i\mid \mathrm{do}(x_j(0)=\xI_j(0)))=\mathbb{P}(x_i\mid \mathrm{do}(x_j(0)=\xII_j(0)))$.
This could be the case because \emph{(i)} $x_i$ and $x_j$ are non-causal or \emph{(ii)} we have that $x\in\mathcal{X}_\mathrm{nc}$ and $u\in\mathcal{U}_\mathrm{nc}$.
Due to assumption~\ref{assume:no_loc_ind} and \eqref{eqn:max_mmd}, there exists a $t$ for which either $x\notin\mathcal{X}_\mathrm{nc}$ or $u\notin\mathcal{U}_\mathrm{nc}$.
Thus, we are in case \emph{(i)} and $x_j$ does not cause $x_i$.
The proof for $u_j$ and $x_i$ follows analogously.
\end{proof}
\begin{remark}
Assumption~\ref{assume:no_loc_ind} ensures that, if we design experiments using the optimization algorithm in~\eqref{eqn:max_mmd}, we are not inside a region of local non-causality.
However, even if we weaken the assumption, the claim of theorem~\ref{thm:general_result} would still hold.
Since \eqref{eqn:max_mmd} finds the experiment design that maximizes the MMD, the algorithm only fails if our model estimate assumes the influence of one variable on the other to be strongest in regions where there is actually no causal influence.
Given a reasonable model class and excitation signal for the initial model estimation, this should not occur in practice.
We still keep assumption~\ref{assume:no_loc_ind} in its stronger form as it allows us to also make causal inference in practical settings where we may not be able to globally solve \eqref{eqn:max_mmd}.
\end{remark}
\begin{remark}
Instead of doing multiple experiments of fixed length, one might consider doing one long experiment.
This design has several disadvantages.
First, if the process is stable, the probability distribution can become independent of the initial conditions over time.
Then, causal influences are only visible during the transient.
Second, if we deal with non-ergodic systems, where time and spatial average are not the same, we only have a valid testing procedure if we do multiple runs of each experiment.
Lastly, doing multiple experiments ensures that the different runs are independent of each other.
Data generated by a single long experiment can be correlated, and obtaining a valid test is way more involved \citep{solowjow2020kernel}.
\end{remark}

\subsection{\texorpdfstring{$\epsilon$}{e}-Controllability}
\label{sec:e-controllability}
For a stochastic system as in \eqref{eqn:gen_sys_nl}, it is in general impossible to steer the system exactly to the initial conditions suggested by \eqref{eqn:max_mmd}; \ie we need to resort to controllability with $\epsilon>0$ (\cf definition~\ref{def:eps_controllability}). 
Nevertheless, even in such cases, it is still possible to guarantee the consistency of the causality testing procedure. 
However, we need a stricter definition of controllability.
\begin{defi}
\label{def:distr_controllability}
Let the system given in \eqref{eqn:gen_sys_nl} be $\epsilon$-controllable according to definition~\ref{def:eps_controllability}, and consider some arbitrary, but fixed $x^*_{\ell,\mathrm{des}}$.  
Then, the system given in \eqref{eqn:gen_sys_nl} is said to be \emph{completely $\epsilon$-controllable in distribution} if, for any $x(0)$ and any $x_\mathrm{des}$ with $x_{\ell,\mathrm{des}}=x^*_{\ell,\mathrm{des}}$, there exists an input sequence $u(0\,:\,t_\mathrm{f})$ such that $x_\ell(t_\mathrm{f})$ always follows the same distribution; 
\ie $\mathbb{P}(x_\ell(t_\mathrm{f}))=\mathbb{P}^*$ for some $\mathbb{P}^*$ that does not depend on $x(0)$ or any component of $x_\mathrm{des}$ except $x_{\ell,\mathrm{des}}=x^*_{\ell,\mathrm{des}}$ and $\mathrm{Pr}\{\norm{x(t_\mathrm{f})-x_\mathrm{des}}_2^2\ge\epsilon\}\le 1-\eta$, with $0<\eta<1$.
\end{defi}
In other words, the definition states that, for any $x(0)$, we can generate input trajectories that guarantee that the fixed component $x^*_{\ell,\mathrm{des}}$ of $x_\mathrm{des}$ is matched in distribution.
Linear systems with additive Gaussian noise that are controllable following \cite{kalman1960contributions} are also controllable in the sense of definition~\ref{def:distr_controllability}, as we show in the appendix.
We further need to make an assumption about the initial conditions suggested by \eqref{eqn:max_mmd}.
To steer the system to those initial conditions without ending up in a region of local non-causality, we need that states that are $\epsilon$-close to those initial conditions are not in a region of local non-causality.
\begin{assume}
\label{assume:local_ind}
For the initial conditions suggested by \eqref{eqn:max_mmd}, we assume that there exist open balls with radius $\sqrt{\epsilon}$ centered around each element of $\xI(0)$ and $\xII(0)$ that are outside of the local non-causality sets.
\end{assume}
This assumption is not very strong since \eqref{eqn:max_mmd} suggests initial conditions for which the influence of $x_j$ respectively $u_j$ on $x_i$ is maximal.
Thus, it is unlikely that these initial conditions are close to regions of local non-causality.
The main reason why this assumption is needed is to exclude corner cases in which the causal influence only exists in single points of the state and input space. 
In such cases, the probability of successfully steering the system to those points is zero.
However, considering the variables as non-causal may then anyway be reasonable.
Equipped with these three assumptions, we can now state:
\begin{cor}
\label{cor:general_result_ext}
Consider a system as in \eqref{eqn:gen_sys_nl} that is completely $\epsilon$-controllable in distribution according to definition~\ref{def:distr_controllability} and fulfills assumptions~\ref{assume:no_loc_ind},~\ref{ass:delay}, and~\ref{assume:local_ind}.
Let experiments be designed as in \eqref{eqn:max_mmd} for a fixed experiment length $T$, trajectories that steer the system to the initial conditions of the experiments be chosen such that $P(\xI_{\ell}(0))\! =\! P(\xII_{\ell}(0))$ for all $\ell\neq j$, and experiments be repeated infinitely often. 
Then: $\mmd(\textbf{x}^\mathrm{I}_i,\textbf{x}^\mathrm{II}_i)=0$ if, and only if, $x_j$ respectively $u_j$ does not cause $x_i$ according to Definition~\ref{def:independence}.
\end{cor}
\begin{proof}
Let variables be non-causal. 
Then, we have $\mathbb{P}(\textbf{x}^\mathrm{I}_{\ell}(0)) = \mathbb{P}(\textbf{x}^\mathrm{II}_{\ell}(0))$ for all $\ell\neq j$, thus, also the distribution of the obtained $x_i$ trajectories is equal and we have $\mmd(\textbf{x}^\mathrm{I}_i,\textbf{x}^\mathrm{II}_i)=0$ \citep{gretton2012kernel}.
Now, assume $\mmd(\textbf{x}^\mathrm{I}_i,\textbf{x}^\mathrm{II}_i)=0$.
This implies that the distribution of $x_i$ is equal in both experiments \citep{gretton2012kernel}. 
By assumption~\ref{assume:no_loc_ind}, existing local non-causalities are reflected by the model and thus,  \eqref{eqn:max_mmd} will suggest experiments outside of such regions.
Assumption~\ref{assume:local_ind} ensures that we can steer the system to those regions.
Thus, if distributions are equal, non-causality must be global as in definition~\ref{def:independence}.
\end{proof} 

\subsection{Linear Systems}

Local non-causality as in definition~\ref{def:loc_ind} is a nonlinear phenomenon.
If we assume \eqref{eqn:gen_sys_nl} to be linear time-invariant (LTI) with Gaussian noise $v(t)$, we can reveal the true causal structure without the optimization procedure in \eqref{eqn:max_mmd}, making this case substantially easier.
For an LTI system, \eqref{eqn:gen_sys_nl} reads
\begin{align}
\label{eqn:loc_sys_lin}
 x(t) = A^tx(0)+ \sum_{i=0}^{t-1} A^i(Bu(t-1-i) + v(t-1-i)),
\end{align}
with state transition matrix $A\in\R^{n\times n}$, input matrix $B\in\R^{n\times m}$, and $v(t)\sim\mathcal{N}(0,\Sigma_\mathrm{v})$.
The system in \eqref{eqn:loc_sys_lin} is controllable as per definition~\ref{def:distr_controllability} if it satisfies the classical controllability condition from \cite{kalman1960contributions}, \ie if the matrix $\begin{pmatrix}B&AB&\ldots&A^{n-1}B\end{pmatrix}$ has full row rank, as we show in lemma~\ref{lem:controllability_linear} in the appendix.
We can then state the following theorem, whose proof is provided in the appendix:
\begin{theo}
\label{thm:linear_result}
Assume an LTI system as in \eqref{eqn:loc_sys_lin}, whose $A$ and $B$ matrices satisfy Kalman's controllability condition.
Let experiments be designed as in \eqref{eqn:exp_design_state} and \eqref{eqn:exp_design_inp}, respectively.
Then:
$\mmd(\textbf{x}^\mathrm{I}_i,\textbf{x}^\mathrm{II}_i)=0$ if, and only if, $x_j$ respectively $u_j$ does not cause $x_i$ as per definition~\ref{def:independence}.
\end{theo}

\subsection{Test with Finite Samples}
\label{sec:finite_samples}

Until here, we derived guarantees in the infinite sample limit.
In practice, we can only carry out finitely many experiments, \ie we have finitely many samples of the random variable sequence $x_i$.
Thus, we need a finite sample approximation of the MMD.
\begin{lem}
\label{lem:finite_sample_approx}
Consider $m$ experiments with fixed length $T$, \ie $\textbf{x}_i\in\R^{m\times T}$ but now we also have $m<\infty$.
An unbiased empirical estimate of the squared population MMD can be computed as
\begin{align}
\label{eqn:mmd_kernel}
\mmdsq(\textbf{x}^\mathrm{I}_i,\textbf{x}^\mathrm{II}_i) = \frac{1}{m(m-1)}\sum_{r\neq s}^m(k(\prescript{r}{}{x}_i^{\mathrm{I}},\prescript{s}{}{x}_i^{\mathrm{II}})+k(\prescript{r}{}{x}_i^{\mathrm{II}},\prescript{s}{}{x}_i^{\mathrm{II}})
-k(\prescript{r}{}{x}_i^{\mathrm{I}},\prescript{s}{}{x}_i^{\mathrm{II}})-k(\prescript{s}{}{x}_i^{\mathrm{I}},\prescript{r}{}{x}_i^{\mathrm{II}})),
\end{align}
where $\prescript{r}{}{x}_i^{\mathrm{I}}$ denotes element $r$ of the $x_i$ trajectories from experiment~I.
\end{lem}
\begin{proof}
The $m$ trajectories for $\xI$ and $\xII$ are independent and follow $\mathbb{P}^\mathrm{I}$ and $\mathbb{P}^\mathrm{II}$, respectively.
Thus, we are in the same setting as in \cite[lem.~6]{gretton2012kernel} and the proof follows as shown therein.
\end{proof}
For a finite sample approximation, we, in general, have $\mathrm{MMD}^2(\textbf{x}_i^\mathrm{I},\textbf{x}_i^\mathrm{II})>0$ even if $\mathbb{P}^\mathrm{I}=\mathbb{P}^\mathrm{II}$.\footnote{Note that the unbiased estimate in lemma~\ref{lem:finite_sample_approx} can even become negative, \cf the discussion after lemma~6 in \cite{gretton2012kernel}.}
Thus, we need to do a hypothesis test and derive a test statistic.
We assume the null hypothesis
\begin{align}
\label{eqn:null}
H_0:\quad \mathbb{P}^\mathrm{I}=\mathbb{P}^\mathrm{II}
\end{align}
and obtain the test statistic from the following result.
\begin{lem}
\label{lem:test_stat}
Assume $0\le k(\prescript{r}{}{x}_i^{\mathrm{I}},\prescript{s}{}{x}_i^{\mathrm{II}})\le K$. Then
\[
\mathrm{Pr}_{\textbf{x}_i^\mathrm{I},\textbf{x}_i^\mathrm{II}}\{\mmdsq(\textbf{x}_i^\mathrm{I},\textbf{x}_i^\mathrm{II})-{\mmd}^2(\mathbb{P}^\mathrm{I},\mathbb{P}^\mathrm{II})>\gamma\}\le\exp\left(\frac{-\gamma^2 m_2}{8K^2}\right),
\]
where $m_2\coloneqq \lfloor m/2\rfloor$. The hypothesis test of level $\alpha$ for the null hypothesis in \eqref{eqn:null} has the acceptance region
\[
       \mmdsq(\textbf{x}^\mathrm{I}_i,\textbf{x}^\mathrm{II}_i) < \left(\frac{4K}{\sqrt{m}}\sqrt{\log(\alpha^{-1})}\right).
\]
\end{lem}
\begin{proof}
The $m$ trajectories for $\xI$ and $\xII$ are independent and follow $\mathbb{P}^\mathrm{I}$ and $\mathbb{P}^\mathrm{II}$, respectively.
Thus, the proof follows from theorem~10 and corollary~11 in \cite{gretton2012kernel}.
\end{proof}

\section{Implementation}
\label{sec:implementation}
The results in \secref{sec:tec_res} show that we are able to detect whether the variables $x_j$ or $u_j$ have a causal influence on $x_i$.
In practical implementations, two challenges remain: first, we want to minimize the number of experiments we need to carry out on the physical platform. 
Second, we may be unable to obtain a global solution to \eqref{eqn:max_mmd}.

\subsection{Heuristic Test with Finite Samples}
\label{sec:heuristic test}
The test statistic provided in lemma~\ref{lem:test_stat} enjoys a theoretical foundation.
However, the threshold decreases as $m^{-\sfrac{1}{2}}$, \ie we need many experiments to not be overly conservative.
While more efficient test statistics exist (\eg achieved through subsampling as discussed in section~6 of \citet{gretton2012kernel}), generating all this data through experiments on real-world systems is often undesirable, \eg because it is time-consuming and may cause excessive wear and tear on the hardware.
Thus, we propose an alternative test statistic that can be obtained from the model estimate $\hat{f}$.
This alternative test statistic is heuristic and forgoes the theoretical properties but is efficient to implement and yields good results in practice, as we show in \secref{sec:eval}.

We estimate a model $\hat{f}_{i,\mathrm{nc}}$ that assumes $x_i$ and $x_j$ respectively $u_j$ are non-causal (\ie we do not use the data $x_j$ respectively $u_j$ when estimating $\hat{f}_{i,\mathrm{nc}}$).
We propose to use this model to decide whether to accept the null hypothesis of $x_i$ and $x_j$ respectively $u_j$ being non-causal. 
That is, we replace our current model $\hat{f}_i$ for state component $x_i$ with $\hat{f}_{i,\mathrm{nc}}$ if
\begin{align}
\label{eqn:test_statistic}
\begin{split}
\mmdsq(\xI_i,\xII_i) < &\E[\mmdsq(\hat{x}^\mathrm{I}_i,\hat{x}^\mathrm{II}_i\mid\hat{f}_{i,\mathrm{nc}})]
+\nu\sqrt{\Var[\mmdsq(\hat{x}^\mathrm{I}_i,\hat{x}^\mathrm{II}_i\mid\hat{f}_{i,\mathrm{nc}})]}.
\end{split}
\end{align}
Expected value and variance in \eqref{eqn:test_statistic} can be estimated through Monte Carlo simulations.
For these simulations, we use the true initial conditions $\xI(0)$ and $\xII(0)$. 
That way, we account for uncertainty due to unequal initial conditions between experiments.
The significance level of the test can be adjusted through $\nu$ using Chebyshev's inequality \citep{chebyshev1874sur}.

\subsection{Experimental Design}
\label{sec:framework}

The framework for testing causality between state variables is summarized in \algref{alg:ext_approach} (and works analogously for inputs).
After having obtained an initial model (see ll.~1-2 in \algref{alg:ext_approach}), we run \eqref{eqn:max_mmd_state} to find initial conditions and input trajectories for testing non-causality of one specific combination of $x_\ell$ and $x_j$ (l.~6).
The optimization in \eqref{eqn:max_mmd_state} may be arbitrarily complex or even intractable, depending on the chosen model class.
However, finding a global optimum of \eqref{eqn:max_mmd_state} is not necessary.
The goal of the optimization procedure is to avoid regions of local non-causality.
We thus optimize \eqref{eqn:max_mmd_state} until it is above a threshold $\delta_1$ to be confident that we are not in a region of local non-causality.
In practical applications, we can often already achieve this through a reasonable initialization of the optimization problem, \ie by choosing initial conditions for $x_j$ as far apart as possible.
We then run the designed experiment and collect the data (l.~7).
Ideally, we would like to use data from this single experiment to test for the causal influence of $x_j$ on all other state components.
Thus, we check for which $x_i$ the experiment yields an expected MMD that is higher than a second threshold $\delta_2$ and do the hypothesis test for all of those (ll.~9-13). 

\begin{algorithm}
\small
\caption{Pseudocode of the proposed framework.}
\label{alg:ext_approach}
\begin{algorithmic}[1]
\State Excite system with input signal, collect data
\State Obtain $\hat{f}$ through black-box system identification
\For {$x_j$ in $x$}
\State $x_\mathrm{test} = [x_1,\ldots,x_n]$ 
\Comment {\parbox[t]{.41\linewidth}{states to be tested for non-causality}} 
\For {$x_\ell$ in $x_\mathrm{test}$} 
\Comment {\parbox[t]{.41\linewidth}{design experiment to test whether $x_j$ causes $x_\ell$}}
\State Run \eqref{eqn:max_mmd_state} until $ \E[\mmdsq(\hat{x}^\mathrm{I}_\ell,\hat{x}^\mathrm{II}_\ell\mid\hat{f})]>\delta_1$
\State Run causal experiments, collect data
\For {$x_i$ in $x_\mathrm{test}$}
\Comment{\parbox[t]{.41\linewidth}{The experimental designed for $x_j$ and $x_\ell$ might also yield a valid test for other $x$}}
\If {$ \E[\mmdsq(\hat{x}^\mathrm{I}_i,\hat{x}^\mathrm{II}_i\mid\hat{f})]>\delta_2$}
\Comment{\parbox[t]{.41\linewidth}{If the empirical MMD is above the threshold for some other $x_i$, also test that $x_i$}}
\State Obtain $\hat{f}_{i,\mathrm{nc}}$
\State Obtain test statistic via Monte Carlo simulations
\If {\eqref{eqn:test_statistic} holds}  
\Comment{\parbox[t]{.41\linewidth}{ independence test}}
\State $\hat{f}_i=\hat{f}_{i,\mathrm{nc}}$
\EndIf
\State Delete $x_i$ from $x_\mathrm{test}$
\EndIf
\EndFor
\EndFor
\EndFor
\end{algorithmic}
\end{algorithm}

\section{Evaluation}
\label{sec:eval}
We evaluate the framework on three systems.
First, we identify the causal structure of one arm of the robot 
Apollo \citep{kappler2018real} 
shown in \figref{fig:robot} in \secref{sec:robot}.
Then, we demonstrate the causal identification of a simulated quadruple tank process (\cf \figref{fig:4tank}) in \secref{sec:4tank}.
In both cases, we present the general setup and results in the following, while we defer implementation details to the appendix.
Lastly, we discuss a synthetic linear toy example.
We use this example to highlight once more the importance of considering a notion of controllability and to compare our method to a sparse identification and a causal discovery algorithm\footnote{Code for both simulation examples is available at 
\url{https://github.com/baumanndominik/identifying_causal_structure}.
}.
For all experiments, we use a Gaussian kernel to compute the MMD.

\subsection{Robot Experiments}
\label{sec:robot}

We consider kinematic control of the robot; that is, we can command desired angular velocities to the joints, which are then tracked by low-level controllers (taking care, among other things, of the robot dynamics) \citep{o80}.
As measurements, we receive the joint angles.
The goal of the causal identification is to learn which joints influence each other and which joints are influenced by which inputs.
We consider four joints of the robot arm in the following experiments.  
From the robot arm's design, we know its kinematic structure, which is described by $\dot{\phi}_i = u_i$ for each joint, where $\phi_i$ denotes the angle of joint $i$ and $u_i$ the control input.  
That is, we expect each joint velocity $\dot{\phi}_i$ to depend only on the local input $u_i$ and not other variables.  
While the dynamics are approximately linear, we do not rely on this information and are thus in the setting discussed in \secref{sec:e-controllability}.
In the following, we will investigate whether our proposed causal identification can automatically reveal this structure.

\begin{figure}
\centering
\includegraphics[width=0.4\textwidth]{img/robot_caus_1_reduced.png}
\hspace{1cm}
\includegraphics[width=0.4\textwidth]{img/robot_caus_2_reduced.png}
\caption{The robot showing initial postures for two experiments.}
\label{fig:robot}
\end{figure}

Following \algref{alg:ext_approach}, we start by identifying a model $\hat{f}$.
In this experiment, we estimate a linear state-space model.
As expected, the initial model suggests that all joints are linked to each other and all inputs due to spurious correlations.
We then design experiments for causality testing, example trajectories of such experiments are shown in \figref{fig:exp_eval} (left).
The empirical squared MMD of the resulting trajectories is compared with the test statistic.
The trajectories in \figref{fig:exp_eval} (left) already suggest that the experiments are in line with the kinematic model: while the two trajectories of joint 1 for different initial conditions of joint 3 are essentially equal (blue dashed and green dotted lines overlap), the trajectories of joint 3 for different choices of the third input are fundamentally different.
This is also revealed through the proposed causality test.
The middle plots of \figref{fig:exp_eval} show the empirical squared MMD (left-hand side of \eqref{eqn:test_statistic}) and the test threshold (right-hand side of \eqref{eqn:test_statistic}) for the experiments that were conducted to test the influence of the initial conditions of the third joint (top) and of the third input signal (bottom) on all joints.
As can be seen, the causal identification reveals that the third joint does not influence any other joint, and the third input only affects the third joint.
Note that the third joint's trajectories are obviously different when choosing different initial conditions for the third joint.
However, since this is expected, we subtract the initial condition in this case to investigate whether the movement starting from these distinct initial conditions differs.
The remaining experiments (results are contained in the appendix) yield similar results.
In summary, the causal identification successfully reveals the expected causal structure.

The experiment design in \eqref{eqn:max_mmd} requires us to solve an optimization problem.
Nevertheless, \algref{alg:ext_approach} introduces two parameters, $\delta_1$ and $\delta_2$.
These can be used to stop the optimization early in case the predicted MMD is high enough for us to be confident that we are not in a region of local non-causality ($\delta_1$) and that we can use the design to test for all influences of a joint or input signal ($\delta_2$).
To design $\delta_1$ and $\delta_2$, we can look at the system's noise level and choose them some orders of magnitude higher.
As discussed in \secref{sec:implementation}, a high predicted MMD can often already be achieved through sophisticated guesses, \ie by choosing initial conditions far apart from each other and diverse input signals.
We follow this approach. 
When testing for the influence of the third joint on others and choosing initial conditions that are far apart, we predict MMDs of around $0.5$.
The model we estimate initially for the robot arm has a noise standard deviation below \num{1e-4}.
That is, the predicted MMD is way above the noise level of the system and way above the MMDs we find in experiments.
Thus, for any choice of $\delta_1$ and $\delta_2$ below 0.5, we can confidently accept this experiment design.
If we were even more conservative and chose them above 0.5, we would need to optimize the experiment design further.

To investigate the generalization capability, we compare predictions of the model $\hat{f}_\mathrm{init}$ obtained from the initial system identification and the model $\hat{f}_\mathrm{caus}$ that was learned after revealing the causal structure.
We use the same training data to estimate the model parameters in both cases. 
However, for $\hat{f}_\mathrm{caus}$, we leverage the obtained knowledge of the causal structure when estimating parameters. In contrast, for $\hat{f}_\mathrm{init}$ we do not take any prior knowledge into account.
As test data, we use an experiment that was conducted to investigate the influence of the initial condition of joint 3 on the other joints and let both models predict the trajectory of joint 1.
For this experiment, the initial angle of joint 3 is close to its maximum value, a case that is not contained in the training data.
As can be seen in \figref{fig:exp_eval} (right), the predictions of $\hat{f}_\mathrm{caus}$ (blue) are very close to the true data (green, dashed), \ie the model can generalize well, while the predictions of $\hat{f}_\mathrm{init}$ deviate significantly.

\begin{figure}
\centering
\begin{tabular}{c|c|c}
\scriptsize{Trajectories from causality experiment}&\scriptsize{Causality test}&\scriptsize{Predictions with learned model}\\
\begin{tikzpicture}

\begin{groupplot}
\nextgroupplot[
 enlargelimits=false,
 scale only axis,
 scale=1,
 width=0.2\textwidth,
 height=0.12\textheight,
 label style={font=\scriptsize},
 tick label style={font=\scriptsize},
 axis lines=left,
 grid=both,
 ylabel = Angle Joint 1,
 xlabel = $t$,
 unit markings = parenthesis,
 y unit=rad, 
]
\addplot [semithick, blue, dashed]
table {%
0 1.50144
1 1.50144
2 1.50539999771313
3 1.50939997854142
4 1.5133999008405
5 1.51739968397551
6 1.52139919571107
7 1.52539823959549
8 1.52939654235576
9 1.53339374131063
10 1.53738937181195
11 1.54138285472791
12 1.54537348398527
13 1.54936041419201
14 1.55334264836664
15 1.55731902580512
16 1.56128821012221
17 1.56524867750979
18 1.56919870526103
19 1.57313636061615
20 1.57705948999216
21 1.58096570866684
22 1.5848523909944
23 1.58871666123829
24 1.5925553851146
25 1.59636516214768
26 1.60014231894772
27 1.60388290352795
28 1.60758268078752
29 1.61123712929338
30 1.61484143950204
31 1.61839051356918
32 1.62187896690075
33 1.62530113160539
34 1.62865106201187
35 1.63192254241912
36 1.63510909724819
37 1.63820400376632
38 1.64120030755193
39 1.64409084086638
40 1.64686824409319
41 1.64952499039739
42 1.65205341374798
43 1.65444574043262
44 1.65669412417815
45 1.65879068497025
46 1.66072755164294
47 1.66249690828154
48 1.66409104445195
49 1.66550240923427
50 1.66672366899974
51 1.66774776882691
52 1.66856799740497
53 1.66917805522103
54 1.66957212577144
55 1.6697449494781
56 1.6696918999264
57 1.66940906197483
58 1.66889331121607
59 1.66814239419687
60 1.6671550087301
61 1.66593088355665
62 1.66447085653992
63 1.66277695050087
64 1.66085244572929
65 1.65870194813777
66 1.65633145196051
67 1.65374839584076
68 1.65096171110062
69 1.64798186094612
70 1.64482086933128
71 1.64149233818924
72 1.63801145173757
73 1.63439496658243
74 1.63066118638136
75 1.62682991988262
76 1.62292242123859
77 1.61896131159603
78 1.61497048109711
79 1.61097497058362
80 1.60700083248415
81 1.60307497057986
82 1.59922495859013
83 1.5954788377938
84 1.59186489420381
85 1.58841141614288
86 1.58514643342087
87 1.58209743969034
88 1.57929109995003
89 1.57675294557293
90 1.57450705965115
91 1.57257575586722
92 1.57097925451426
93 1.56973535968767
94 1.56885914205002
95 1.56836263191961
96 1.56825452774129
97 1.5685399252559
98 1.56922007288087
99 1.57029215893817
};
\addplot [semithick, green!50!black, dotted]
table {%
0 1.50144
1 1.50144
2 1.50539999771313
3 1.50939997854142
4 1.5133999008405
5 1.51739968397551
6 1.52139919571107
7 1.52539823959549
8 1.52939654235576
9 1.53339374131063
10 1.53738937181195
11 1.54138285472791
12 1.54537348398527
13 1.54936041419201
14 1.55334264836664
15 1.55731902580512
16 1.56128821012221
17 1.56524867750979
18 1.56919870526103
19 1.57313636061615
20 1.57705948999216
21 1.58096570866684
22 1.5848523909944
23 1.58871666123829
24 1.5925553851146
25 1.59636516214768
26 1.60014231894772
27 1.60388290352795
28 1.60758268078752
29 1.61123712929338
30 1.61484143950204
31 1.61839051356918
32 1.62187896690075
33 1.62530113160539
34 1.62865106201187
35 1.63192254241912
36 1.63510909724819
37 1.63820400376632
38 1.64120030755193
39 1.64409084086638
40 1.64686824409319
41 1.64952499039739
42 1.65205341374798
43 1.65444574043262
44 1.65669412417815
45 1.65879068497025
46 1.66072755164294
47 1.66249690828154
48 1.66409104445195
49 1.66550240923427
50 1.66672366899974
51 1.66774776882691
52 1.66856799740497
53 1.66917805522103
54 1.66957212577144
55 1.6697449494781
56 1.6696918999264
57 1.66940906197483
58 1.66889331121607
59 1.66814239419687
60 1.6671550087301
61 1.66593088355665
62 1.66447085653992
63 1.66277695050087
64 1.66085244572929
65 1.65870194813777
66 1.65633145196051
67 1.65374839584076
68 1.65096171110062
69 1.64798186094612
70 1.64482086933128
71 1.64149233818924
72 1.63801145173757
73 1.63439496658243
74 1.63066118638136
75 1.62682991988262
76 1.62292242123859
77 1.61896131159603
78 1.61497048109711
79 1.61097497058362
80 1.60700083248415
81 1.60307497057986
82 1.59922495859013
83 1.5954788377938
84 1.59186489420381
85 1.58841141614288
86 1.58514643342087
87 1.58209743969034
88 1.57929109995003
89 1.57675294557293
90 1.57450705965115
91 1.57257575586722
92 1.57097925451426
93 1.56973535968767
94 1.56885914205002
95 1.56836263191961
96 1.56825452774129
97 1.5685399252559
98 1.56922007288087
99 1.57029215893817
};
\end{groupplot}

\end{tikzpicture}&
\begin{tikzpicture}
\begin{groupplot}[
  group style={
     group name=second,
     group size=1 by 2,
     vertical sep=0pt,
     x descriptions at=edge bottom},
     enlargelimits = false,
     scale only axis,
     scale=1, 
     width=0.2\textwidth,
     xmin=0,xmax=4.5,
     tick label style={font=\scriptsize},
     label style={font=\scriptsize, xshift=-0.5cm},
     axis lines = left,
     grid=both,
     anchor=north west]

\nextgroupplot[         
 ylabel=$\widehat{\text{MMD}}^2$,
 x axis line style={draw=none},
 x tick style={draw=none},
 axis y discontinuity=parallel,
 ymin=4,
 ytick={5.5},
 yticklabels={5.5e-7},
 xtick={0,1,2,3,4},
 height=0.06\textheight]
\addplot[very thick,black,dashed] table{%
1 5.80702210
2 4.53893178
3 6.15602498
4 5.20301166
 };

\nextgroupplot[
 xlabel={\# Joint},
 axis x line=bottom,
 ytick={-1,0,1},
 yticklabels={-1e-16,0,1e-16},
 xtick={0,1,2,3,4},
 legend style={font=\scriptsize, at={(0.8,0.5)}},
 height=0.06\textheight,]
\addplot[only marks,red,mark=asterisk, mark size=2.5000pt] table{%
 1 0
 2 0
 3 -1.38
 4 1.29
 };

\end{groupplot}
\end{tikzpicture}&
\begin{tikzpicture}

\definecolor{color0}{rgb}{0.12156862745098,0.466666666666667,0.705882352941177}
\definecolor{color1}{rgb}{1,0.498039215686275,0.0549019607843137}
\definecolor{color2}{rgb}{0.172549019607843,0.627450980392157,0.172549019607843}

\begin{groupplot}
\nextgroupplot[
 enlargelimits=false,
 scale only axis,
 scale=1,
 width=0.2\textwidth,
 height=0.12\textheight,
 label style={font=\scriptsize},
 tick label style={font=\scriptsize},
 ymax=10,
 axis lines=left,
 grid=both,
 ylabel = Angle Joint 1,
 unit markings = parenthesis,
 y unit=rad, 
 xlabel= $t$
]
\addplot [blue]
table {%
0 1.50144
1 1.01470189750577
2 0.528550794056191
3 0.0429857499462985
4 -0.441994207729863
5 -0.926390097692413
6 -1.41020299708405
7 -1.89343405406944
8 -2.37608450041922
9 -2.85815566406936
10 -3.33964898164391
11 -3.82056601092508
12 -4.30090844325096
13 -4.78067811581652
14 -5.25987702384864
15 -5.73850733262056
16 -6.21657138926536
17 -6.69407173434193
18 -7.1710111131002
19 -7.64739248638579
20 -8.12321904111667
21 -8.59849420025694
22 -9.07322163220545
23 -9.54740525950843
24 -10.0210492667978
25 -10.494158107848
26 -10.9667365116371
27 -11.4387894872887
28 -11.9103223277638
29 -12.3813406121653
30 -12.8518502065086
31 -13.3218572628076
32 -13.7913682163175
33 -14.2603897807732
34 -14.7289289414559
35 -15.1969929459194
36 -15.6645892922068
37 -16.1317257143867
38 -16.5984101652432
39 -17.0646507959562
40 -17.5304559326157
41 -17.9958340494248
42 -18.460793738455
43 -18.9253436758365
44 -19.3894925842801
45 -19.8532491918536
46 -20.3166221869573
47 -20.7796201694753
48 -21.2422515981106
49 -21.7045247339515
50 -22.1664475803569
51 -22.6280278192938
52 -23.0892727443142
53 -23.5501891904062
54 -24.0107834610215
55 -24.4710612526344
56 -24.9310275772618
57 -25.3906866834366
58 -25.8500419762017
59 -26.3090959367642
60 -26.7678500425262
61 -27.2263046882802
62 -27.6844591094377
63 -28.1423113082268
64 -28.5998579838706
65 -29.0570944678206
66 -29.5140146651819
67 -29.9706110035199
68 -30.4268743902826
69 -30.8827941801056
70 -31.3383581532892
71 -31.7935525067417
72 -32.2483618586735
73 -32.702769268298
74 -33.1567562717448
75 -33.6103029353202
76 -34.0633879271536
77 -34.5159886081459
78 -34.9680811429892
79 -35.4196406318481
80 -35.8706412630892
81 -36.321056487209
82 -36.7708592118459
83 -37.220022017471
84 -37.6685173930312
85 -38.1163179904719
86 -38.5633968967008
87 -39.0097279211636
88 -39.4552858968024
89 -39.9000469917517
90 -40.34398902871
91 -40.7870918085103
92 -41.2293374340042
93 -41.6707106299885
94 -42.1111990545417
95 -42.5507935968113
96 -42.9894886560189
97 -43.4272823962262
98 -43.8641769712579
99 -44.3001787141013
100 -44.7352982851243
};
\addplot [green!50!black, dashed]
table {%
0 1.50144
1 1.50144
2 1.50539999771313
3 1.50939997854142
4 1.5133999008405
5 1.51739968397551
6 1.52139919571107
7 1.52539823959549
8 1.52939654235576
9 1.53339374131063
10 1.53738937181195
11 1.54138285472791
12 1.54537348398527
13 1.54936041419201
14 1.55334264836664
15 1.55731902580512
16 1.56128821012221
17 1.56524867750979
18 1.56919870526103
19 1.57313636061615
20 1.57705948999216
21 1.58096570866684
22 1.5848523909944
23 1.58871666123829
24 1.5925553851146
25 1.59636516214768
26 1.60014231894772
27 1.60388290352795
28 1.60758268078752
29 1.61123712929338
30 1.61484143950204
31 1.61839051356918
32 1.62187896690075
33 1.62530113160539
34 1.62865106201187
35 1.63192254241912
36 1.63510909724819
37 1.63820400376632
38 1.64120030755193
39 1.64409084086638
40 1.64686824409319
41 1.64952499039739
42 1.65205341374798
43 1.65444574043262
44 1.65669412417815
45 1.65879068497025
46 1.66072755164294
47 1.66249690828154
48 1.66409104445195
49 1.66550240923427
50 1.66672366899974
51 1.66774776882691
52 1.66856799740497
53 1.66917805522103
54 1.66957212577144
55 1.6697449494781
56 1.6696918999264
57 1.66940906197483
58 1.66889331121607
59 1.66814239419687
60 1.6671550087301
61 1.66593088355665
62 1.66447085653992
63 1.66277695050087
64 1.66085244572929
65 1.65870194813777
66 1.65633145196051
67 1.65374839584076
68 1.65096171110062
69 1.64798186094612
70 1.64482086933128
71 1.64149233818924
72 1.63801145173757
73 1.63439496658243
74 1.63066118638136
75 1.62682991988262
76 1.62292242123859
77 1.61896131159603
78 1.61497048109711
79 1.61097497058362
80 1.60700083248415
81 1.60307497057986
82 1.59922495859013
83 1.5954788377938
84 1.59186489420381
85 1.58841141614288
86 1.58514643342087
87 1.58209743969034
88 1.57929109995003
89 1.57675294557293
90 1.57450705965115
91 1.57257575586722
92 1.57097925451426
93 1.56973535968767
94 1.56885914205002
95 1.56836263191961
96 1.56825452774129
97 1.5685399252559
98 1.56922007288087
99 1.57029215893817
};
\end{groupplot}

\end{tikzpicture}\\
\begin{tikzpicture}

\begin{groupplot}
\nextgroupplot[
enlargelimits=false,
 scale only axis,
 scale=1,
 width=0.2\textwidth,
 height=0.12\textheight,
 label style={font=\scriptsize},
 tick label style={font=\scriptsize},
 legend style={font=\scriptsize, at={(0.9,0.65)}},
 axis lines=left,
 grid=both,
 ylabel = Angle Joint 1,
 unit markings = parenthesis,
 y unit=rad, 
 xlabel= $t$
]
\addplot [semithick, blue, dashed]
table {%
0 0.00144
1 0.00144
2 0.00539999771313287
3 0.00939997854142432
4 0.0133999008405039
5 0.0173996839755139
6 0.0213991957110691
7 0.0253982395954893
8 0.0293965423557641
9 0.033393741310628
10 0.0373893718119475
11 0.0413828547279127
12 0.0453734839852672
13 0.0493604141920088
14 0.0533426483666384
15 0.057319025805118
16 0.0612882101222091
17 0.0652486775097856
18 0.0691987052610349
19 0.0731363606161506
20 0.0770594899921554
21 0.0809657086668376
22 0.084852390994404
23 0.088716661238292
24 0.0925553851145957
25 0.0963651621476811
26 0.100142318947716
27 0.103882903527948
28 0.107582680787524
29 0.111237129293376
30 0.114841439502042
31 0.118390513569175
32 0.121878966900746
33 0.125301131605386
34 0.128651062011871
35 0.13192254241912
36 0.135109097248192
37 0.138204003766322
38 0.141200307551926
39 0.144090840866383
40 0.146868244093188
41 0.149524990397394
42 0.152053413747975
43 0.154445740432619
44 0.156694124178154
45 0.158790684970251
46 0.160727551642937
47 0.162496908281538
48 0.164091044451952
49 0.165502409234268
50 0.166723668999743
51 0.167747768826904
52 0.168567997404969
53 0.169178055221024
54 0.169572125771436
55 0.169744949478096
56 0.169691899926397
57 0.169409061974832
58 0.168893311216066
59 0.168142394196867
60 0.167155008730095
61 0.165930883556645
62 0.164470856539917
63 0.162776950500867
64 0.160852445729285
65 0.158701948137771
66 0.156331451960512
67 0.15374839584076
68 0.150961711100619
69 0.147981860946113
70 0.144820869331283
71 0.141492338189234
72 0.138011451737571
73 0.134394966582426
74 0.130661186381356
75 0.126829919882618
76 0.122922421238588
77 0.118961311596032
78 0.114970481097109
79 0.110974970583622
80 0.107000832484153
81 0.103074970579854
82 0.0992249585901307
83 0.0954788377937949
84 0.0918648942038081
85 0.088411416142877
86 0.0851464334208628
87 0.0820974396903378
88 0.0792910999500311
89 0.0767529455729297
90 0.0745070596511461
91 0.0725757558672153
92 0.070979254514259
93 0.0697353596876672
94 0.068859142050019
95 0.0683626319196105
96 0.0682545277412874
97 0.0685399252559001
98 0.0692200728808664
99 0.070292158938164
};
\addlegendentry{Exp. I};
\addplot [semithick, green!50!black, dotted]
table {%
0 -0.00144
1 -0.00144
2 -0.00539999771313287
3 -0.00939997854142432
4 -0.0133999008405039
5 -0.0173996839755139
6 -0.0213991957110691
7 -0.0253982395954893
8 -0.0293965423557641
9 -0.033393741310628
10 -0.0373893718119475
11 -0.0413828547279127
12 -0.0453734839852672
13 -0.0493604141920088
14 -0.0533426483666384
15 -0.057319025805118
16 -0.0612882101222091
17 -0.0652486775097856
18 -0.0691987052610349
19 -0.0731363606161506
20 -0.0770594899921554
21 -0.0809657086668376
22 -0.084852390994404
23 -0.088716661238292
24 -0.0925553851145957
25 -0.0963651621476811
26 -0.100142318947716
27 -0.103882903527948
28 -0.107582680787524
29 -0.111237129293376
30 -0.114841439502042
31 -0.118390513569175
32 -0.121878966900746
33 -0.125301131605386
34 -0.128651062011871
35 -0.13192254241912
36 -0.135109097248192
37 -0.138204003766322
38 -0.141200307551926
39 -0.144090840866383
40 -0.146868244093188
41 -0.149524990397394
42 -0.152053413747975
43 -0.154445740432619
44 -0.156694124178154
45 -0.158790684970251
46 -0.160727551642937
47 -0.162496908281538
48 -0.164091044451952
49 -0.165502409234268
50 -0.166723668999743
51 -0.167747768826904
52 -0.168567997404969
53 -0.169178055221024
54 -0.169572125771436
55 -0.169744949478096
56 -0.169691899926397
57 -0.169409061974832
58 -0.168893311216066
59 -0.168142394196867
60 -0.167155008730095
61 -0.165930883556645
62 -0.164470856539917
63 -0.162776950500867
64 -0.160852445729285
65 -0.158701948137771
66 -0.156331451960512
67 -0.15374839584076
68 -0.150961711100619
69 -0.147981860946113
70 -0.144820869331283
71 -0.141492338189234
72 -0.138011451737571
73 -0.134394966582426
74 -0.130661186381356
75 -0.126829919882618
76 -0.122922421238588
77 -0.118961311596032
78 -0.114970481097109
79 -0.110974970583622
80 -0.107000832484153
81 -0.103074970579854
82 -0.0992249585901307
83 -0.0954788377937949
84 -0.0918648942038081
85 -0.088411416142877
86 -0.0851464334208628
87 -0.0820974396903378
88 -0.0792910999500311
89 -0.0767529455729297
90 -0.0745070596511461
91 -0.0725757558672153
92 -0.070979254514259
93 -0.0697353596876672
94 -0.068859142050019
95 -0.0683626319196105
96 -0.0682545277412874
97 -0.0685399252559001
98 -0.0692200728808664
99 -0.070292158938164
};
\addlegendentry{Exp. II};
\end{groupplot}

\end{tikzpicture}&
\begin{tikzpicture}

\begin{axis}[%
 enlargelimits = false,
 scale only axis,
 scale=1,
 width=0.2\textwidth,
 height=0.12\textheight,
 label style={font=\scriptsize},
 tick label style={font=\scriptsize},
 legend style={font=\scriptsize, at={(0.8,0.5)}},
 ylabel=$\widehat{\text{MMD}}^2$,
 xlabel= \# Joint,
 axis lines=left, 
xmin=0,xmax=4.5,
 ymin=1e-14,ymax=15,
 grid=both,
 ymode = log,
 xtick={0,1,2,3,4},
 xticklabels={,1,2,3,4},
]
 \addplot[only marks,red,mark=asterisk, mark size=2.5000pt] table{%
 1 2.31e-13
 2 1.4e-13
 3 0.04
 4 2.25e-12
 };
 \addlegendentry{Exp. MMD};
 \addplot[black,very thick,dashed] table{%
1 5.26239113e-05
2 4.28256910e-05
3 1.17784491e-05
4 4.36026446e-07
 };
\addlegendentry{Test stat.};

\end{axis}
\end{tikzpicture}&
\begin{tikzpicture}

\begin{groupplot}

\nextgroupplot[
 enlargelimits=false,
 scale only axis,
 scale=1,
 width=0.2\textwidth,
 height=0.12\textheight,
 label style={font=\scriptsize},
 tick label style={font=\scriptsize},
 legend style={font=\scriptsize, at={(0.9,0.4)}},
 axis lines=left,
 grid=both,
 xlabel = $t$,
 ylabel = Angle Joint 1,
 y unit=rad, 
]
\addplot [blue]
table {%
0 1.50144
1 1.50644405725358
2 1.5114481065602
3 1.5164521110985
4 1.52145599081078
5 1.52645960661149
6 1.53146274460094
7 1.53646510028967
8 1.54146626284133
9 1.54646569934574
10 1.55146273913711
11 1.55645655817755
12 1.56144616353037
13 1.56643037795388
14 1.57140782465224
15 1.57637691222666
16 1.58133581987754
17 1.58628248291587
18 1.59121457865049
19 1.59612951272624
20 1.60102440599738
21 1.60589608202988
22 1.61074105533604
23 1.61555552045472
24 1.62033534200082
25 1.6250760458175
26 1.6297728113753
27 1.63442046557184
28 1.63901347809603
29 1.64354595852987
30 1.64801165536985
31 1.65240395715806
32 1.65671589592037
33 1.66094015311526
34 1.66506906830134
35 1.66909465073545
36 1.67300859411389
37 1.67680229466955
38 1.68046687283404
39 1.68399319866889
40 1.68737192126113
41 1.69059350226713
42 1.69364825377323
43 1.69652638062296
44 1.69921802733761
45 1.70171332972953
46 1.70400247127568
47 1.70607574428233
48 1.70792361582966
49 1.70953679843859
50 1.71090632534938
51 1.71202363024447
52 1.71288063118448
53 1.71346981845872
54 1.71378434597783
55 1.71381812575823
56 1.71356592496563
57 1.71302346489829
58 1.71218752120103
59 1.71105602450868
60 1.70962816062364
61 1.70790446923788
62 1.70588694011567
63 1.70357910556138
64 1.70098612790855
65 1.69811488068332
66 1.69497402201895
67 1.69157405883198
68 1.6879274002142
69 1.68404839845313
70 1.67995337606697
71 1.67566063723254
72 1.67119046199704
73 1.66656508170105
74 1.66180863410166
75 1.6569470967745
76 1.65200819749353
77 1.64702130043943
78 1.64201726727389
79 1.63702829233816
80 1.63208771149184
81 1.62722978440253
82 1.62248945042709
83 1.61790205859236
84 1.61350307258337
85 1.60932775208001
86 1.60541081224384
87 1.60178606364241
88 1.59848603540321
89 1.59554158490695
90 1.59298149785316
91 1.59083208305216
92 1.58911676680567
93 1.58785569222476
94 1.58706532928613
95 1.58675810183408
96 1.5869420380834
97 1.58762045145351
98 1.5887916587546
99 1.59044874283649
100 1.59257936678887
};
\addlegendentry{Exp. I};
\addplot [green!50!black, dashed]
table {%
0 1.50144
1 1.50144
2 1.50539999771313
3 1.50939997854142
4 1.5133999008405
5 1.51739968397551
6 1.52139919571107
7 1.52539823959549
8 1.52939654235576
9 1.53339374131063
10 1.53738937181195
11 1.54138285472791
12 1.54537348398527
13 1.54936041419201
14 1.55334264836664
15 1.55731902580512
16 1.56128821012221
17 1.56524867750979
18 1.56919870526103
19 1.57313636061615
20 1.57705948999216
21 1.58096570866684
22 1.5848523909944
23 1.58871666123829
24 1.5925553851146
25 1.59636516214768
26 1.60014231894772
27 1.60388290352795
28 1.60758268078752
29 1.61123712929338
30 1.61484143950204
31 1.61839051356918
32 1.62187896690075
33 1.62530113160539
34 1.62865106201187
35 1.63192254241912
36 1.63510909724819
37 1.63820400376632
38 1.64120030755193
39 1.64409084086638
40 1.64686824409319
41 1.64952499039739
42 1.65205341374798
43 1.65444574043262
44 1.65669412417815
45 1.65879068497025
46 1.66072755164294
47 1.66249690828154
48 1.66409104445195
49 1.66550240923427
50 1.66672366899974
51 1.66774776882691
52 1.66856799740497
53 1.66917805522103
54 1.66957212577144
55 1.6697449494781
56 1.6696918999264
57 1.66940906197483
58 1.66889331121607
59 1.66814239419687
60 1.6671550087301
61 1.66593088355665
62 1.66447085653992
63 1.66277695050087
64 1.66085244572929
65 1.65870194813777
66 1.65633145196051
67 1.65374839584076
68 1.65096171110062
69 1.64798186094612
70 1.64482086933128
71 1.64149233818924
72 1.63801145173757
73 1.63439496658243
74 1.63066118638136
75 1.62682991988262
76 1.62292242123859
77 1.61896131159603
78 1.61497048109711
79 1.61097497058362
80 1.60700083248415
81 1.60307497057986
82 1.59922495859013
83 1.5954788377938
84 1.59186489420381
85 1.58841141614288
86 1.58514643342087
87 1.58209743969034
88 1.57929109995003
89 1.57675294557293
90 1.57450705965115
91 1.57257575586722
92 1.57097925451426
93 1.56973535968767
94 1.56885914205002
95 1.56836263191961
96 1.56825452774129
97 1.5685399252559
98 1.56922007288087
99 1.57029215893817
};
\addlegendentry{Exp. II};
\end{groupplot}

\end{tikzpicture}
\end{tabular}
\caption{Causality tests and model evaluation for the robotic system. \capt{Plots on the left show example trajectories of two experiments, in the middle the experimental MMD and the test threshold for joint 3, and on the right predictions based on the initial model and on the refined model after the causal identification.}}
\label{fig:exp_eval}
\end{figure}

\subsection{Quadruple Tank Process}
\label{sec:4tank}

The experimental demonstration in the previous section showed that the presented algorithm can successfully identify a real-world robotic system's causal structure.
However, the causal structure is relatively straightforward, and dynamics are approximately linear.
To stress the method's ability to perform similarly well on more complex structures and with nonlinear dynamics, we now consider the quadruple tank system from \cite{johansson2000quadruple}, illustrated in \figref{fig:4tank}.
Its continuous-time dynamics are given by
\begin{align}
\begin{split}
\label{eqn:4tank}
\dot{x}_1 &= -\frac{a_1}{A_1}\sqrt{2gx_1} + \frac{a_3}{A_1}\sqrt{2gx_3}+\frac{\zeta_1}{A_1}u_1\\
\dot{x}_2 &= -\frac{a_2}{A_2}\sqrt{2gx_2} + \frac{a_3}{A_2}\sqrt{2gx_4}+\frac{\zeta_2}{A_2}u_2\\
\dot{x}_3 &= -\frac{a_3}{A_3}\sqrt{2gx_3} + \frac{1-\zeta_2}{A_3}u_2\\
\dot{x}_4 &= -\frac{a_4}{A_4}\sqrt{2gx_4} + \frac{1-\zeta_1}{A_4}u_1,
\end{split}
\end{align}
where $x_i$ denotes the water level of tank $i$, $u_i$ the flow rate of pump $i$, $g$ the gravitational constant (\SI{9.81}{\meter\per\second\squared}), and the remaining constants as in \figref{fig:4tank}.
The dynamics of the quadruple tank process are, thus, clearly nonlinear.
That is, we cannot expect good performance if we approximate them using a linear state-space model.
Therefore, we discretize the system and use Gaussian processes (GPs) to model the dynamics, see \cite{williams2006gaussian} for an introduction and details.
In particular, we model each state of the system with a GP, where each GP uses all states and inputs to predict its assigned state variable at the next time step.

\begin{figure}
\centering
\input{tikz/4tank.tex}
\caption{Schematic of the quadruple tank process from \cite{johansson2000quadruple}.}
\label{fig:4tank}
\end{figure}

We proceed similarly as for the robot experiments.
We select initial conditions through sophisticated guesses, \ie far apart from each other.
In a simulation, we can set initial conditions and do not need to steer the system there.
However, to be more realistic, we sample initial conditions from a normal distribution with mean equal to the selected initial conditions and variance of \num{1e-2}.
Thus, we only consider $\epsilon$-controllability.

Results of the testing procedure are shown in \figref{fig:res_4tank}.
The left plot shows test statistic and experimental MMD for the third tank's influence on all others. 
It reveals that the third tank influences itself and the first tank, which is in line with the schematic in \figref{fig:4tank} and the dynamics in \eqref{eqn:4tank}.
The right plot illustrates the test statistic and experimental MMD for the influence of $u_1$ on the tanks.
Also these results are in accordance with \figref{fig:4tank}, as the experiments suggest that $u_1$ influences all but the third tank.
The remaining results are collected in \secref{sec:4tank_add} in the appendix.
Similarly as the ones in \figref{fig:res_4tank}, they reveal the causal structure that can be inferred from \figref{fig:4tank}.

\begin{figure}
\centering
\begin{tikzpicture}

\begin{axis}[%
 enlargelimits = false,
 scale only axis,
 scale=1,
 width=0.25\textwidth,
 height=0.15\textheight,
 label style={font=\scriptsize},
 tick label style={font=\scriptsize},
 legend style={font=\scriptsize, at={(1.75,0.75)}},
 ylabel=$\widehat{\text{MMD}}^2$,
 xlabel= \# Tank,
 axis lines=left,
xmin=0,xmax=4.5,
 ymin=1e-6,ymax=1,
 grid=both,
 ymode = log,
 xtick={0,1,2,3,4},
 xticklabels={,1,2,3,4},
]

 \addplot[only marks,red,mark=asterisk, mark size=2.5000pt] table{%
 1 1.61596932e-01
 2 5.36474460e-06
 3 3.44238723e-01
 4 9.58978798e-06
 };
 \addlegendentry{Exp. MMD};
 \addplot[black,very thick,dashed] table{%
1 1.28703266e-04
2 3.28336414e-04
3 2.75153365e-04
4 3.46257839e-05
 };
\addlegendentry{Test stat.};
\end{axis}
\end{tikzpicture}
\begin{tikzpicture}

\begin{axis}[%
 enlargelimits = false,
 scale only axis,
 scale=1,
 width=0.25\textwidth,
 height=0.15\textheight,
 label style={font=\scriptsize},
 tick label style={font=\scriptsize},
 ylabel=$\widehat{\text{MMD}}^2$,
 xlabel= \# Tank,
 axis lines=left,
xmin=0,xmax=4.5,
 ymin=1e-8,ymax=1,
 grid=both,
 ymode = log,
 xtick={0,1,2,3,4},
 xticklabels={,1,2,3,4},
]

 \addplot[only marks,red,mark=asterisk, mark size=2.5000pt] table{%
 1 1.80014841e-01
 2 8.03938714e-03
 3 2.45683662e-08
 4 1.55916573e-01
 };
 \addplot[black,very thick,dashed] table{%
1 1.00558927e-04
2 8.74086114e-05
3 3.80417997e-05
4 8.60622177e-05
 };
\end{axis}
\end{tikzpicture}
\caption{Causality tests for the quadruple tank system. \capt{The plots show the results of the experiments testing the influence of the third tank on all others (left) and the influence of $u_1$ on all tanks (right).}}
\label{fig:res_4tank}
\end{figure}

In this case, the causal identification results in a reduction of computational complexity.
Especially for the third and the fourth tanks, which are only influenced by $u_2$ respectively $u_1$ and by themselves, the input dimension of their corresponding GPs decreases from 10 to 3.
The complexity of standard GP regression grows with $\mathcal{O}(n^3)$ with the number of datapoints $n$.
Thus, if we can reduce the dimensions that need to be considered and, with that, the number of data points that we use for the regression by \SI{70}{\percent} as in this example, we can considerably reduce the computational complexity.

\subsection{Synthetic Example and Comparison}
\label{sec:syn_exp}

Lastly, we present a synthetic, linear example.
We consider an LTI system as in \eqref{eqn:loc_sys_lin}, with
\begin{align}
\label{eqn:syn_example}
A = \begin{pmatrix}
0.9 & -0.75 & 1.2 \\
0 & 0.9 & -1.1 \\
0 & 0 & 0.7
\end{pmatrix}
\quad
B = \begin{pmatrix}
0.03 & 0 & 0 \\
0 & 0.06 & 0 \\
0.07 & 0 & 0.05
\end{pmatrix}
\end{align}
and Gaussian noise with standard deviation \num{1e-4}.
For this example, we apply \algref{alg:ext_approach} without the need for the optimization procedure since the example is linear.
Again, we want to stress the importance of an appropriate notion of controllability.
That is, instead of assuming that we can set the system to initial conditions, we always \emph{steer} the system to the initial conditions required for that experiment.
For this, we employ an approach to set-point tracking that has, for instance, been discussed in \cite{pannocchia2005candidate}.
Given a desired state $x_\mathrm{des}$, we seek a feedback control law of the form $u = Mx_\mathrm{des} + Fx$, \ie a control law that depends both on the desired state and the current state.
We obtain the gain matrix $F$ using standard methods from linear optimal control \citep{anderson2007optimal}, in particular, the linear quadratic regulator (LQR).
Thus, we can rewrite the incremental dynamics of the system as
\begin{align}
x(t+1) = \tilde{A}x(t) + BMx_\mathrm{des} + v(t),
\end{align}
where $\tilde{A} \coloneqq A+BF$.
We now choose the feedforward term $M$ such that the reference is matched in stationarity, \ie we want to achieve
\begin{align}
x = (I - \tilde{A})^{-1}BMx_\mathrm{des}.
\end{align}
Thus, we need
\begin{align}
M = ((I - \tilde{A})^{-1}B)^{-1}
\end{align}
to track the reference point.
To compute $M$, we use the matrices $\hat{A}$ and $\hat{B}$ of the estimated model $\hat{f}$.
We start the experiment once $\norm{x(t)-x_\mathrm{des}}_2 < 0.01$.

Analyzing the resulting data with the MMD lets us, similar as before, infer the true causal structure.
In particular, the causal analysis reveals that $x_1$ does not cause $x_2$ nor $x_3$, $x_2$ does not cause $x_3$, and $u_2$ does not cause $x_3$. 

In addition to those results, we compare our method with a sparse identification method and a causal discovery algorithm for this linear example.
In both cases, we excite the system for \num{100000} time-steps with a chirp as an input signal.

First, we use SINDy, as in example~\ref{exp:loc_ind}, to identify the underlying discrete-time system\footnote{We use the implementation provided in \cite{desilva2020pysindy}.}.
Sparse identification algorithms seek to find a trade-off between model complexity and goodness of fit.
Thus, for such algorithms, a parameter needs to be selected to indicate how important accuracy is compared to complexity.
When applying SINDy to the synthetic system, its success in finding the true causal structure depends on the choice of this parameter.
We start with the parameter settings that were used for a three-dimensional linear example in \cite{brunton2016discovering}.
For those settings, the algorithm does not recognize the influence of $u_1$ on $x_1$, \ie the first row of the $B$ matrix consists of zeros.
In \figref{fig:synt_comp}, we show the effects that this error can have by comparing predictions of the true model obtained after the causal identification with predictions of the model obtained from SINDy.
For this, we set $u_1 = 10$, $u_2=0$, and $u_3 = -14$ and simulate both models for \num{100} time steps. 
As the SINDy model does not reflect the influence of $u_1$ on $x_1$, it assumes that $x_1$ does not move.
However, the model obtained after the causal identification reflects this influence and, thus, correctly predicts the movement of $x_1$.

\begin{figure}
\centering
\begin{tikzpicture}

\begin{axis}[
 enlargelimits=false,
 scale only axis,
 scale=1,
 width=0.3\textwidth,
 height=0.17\textheight,
 label style={font=\scriptsize},
 tick label style={font=\scriptsize},
 legend style={font=\scriptsize, at={(0.9,0.6)}},
 ymax=3.5, ymin=-0.5,
 axis lines=left,
 grid=both,
 ylabel = $x_1$,
 unit markings = parenthesis,
 xlabel= $t$
]
\addplot [blue, thick]
table {%
0 0
1 0.3
2 0.57
3 0.813
4 1.0317
5 1.22853
6 1.405677
7 1.5651093
8 1.70859837
9 1.837738533
10 1.9539646797
11 2.05856821173
12 2.152711390557
13 2.2374402515013
14 2.31369622635117
15 2.38232660371605
16 2.44409394334445
17 2.49968454901
18 2.549716094109
19 2.5947444846981
20 2.63527003622829
21 2.67174303260546
22 2.70456872934492
23 2.73411185641042
24 2.76070067076938
25 2.78463060369244
26 2.8061675433232
27 2.82555078899088
28 2.84299571009179
29 2.85869613908261
30 2.87282652517435
31 2.88554387265692
32 2.89698948539122
33 2.9072905368521
34 2.91656148316689
35 2.9249053348502
36 2.93241480136518
37 2.93917332122866
38 2.9452559891058
39 2.95073039019522
40 2.9556573511757
41 2.96009161605813
42 2.96408245445231
43 2.96767420900708
44 2.97090678810637
45 2.97381610929574
46 2.97643449836616
47 2.97879104852955
48 2.98091194367659
49 2.98282074930893
50 2.98453867437804
51 2.98608480694024
52 2.98747632624621
53 2.98872869362159
54 2.98985582425943
55 2.99087024183349
56 2.99178321765014
57 2.99260489588512
58 2.99334440629661
59 2.99400996566695
60 2.99460896910026
61 2.99514807219023
62 2.99563326497121
63 2.99606993847409
64 2.99646294462668
65 2.99681665016401
66 2.99713498514761
67 2.99742148663285
68 2.99767933796956
69 2.99791140417261
70 2.99812026375535
71 2.99830823737981
72 2.99847741364183
73 2.99862967227765
74 2.99876670504988
75 2.99889003454489
76 2.9990010310904
77 2.99910092798136
78 2.99919083518323
79 2.9992717516649
80 2.99934457649841
81 2.99941011884857
82 2.99946910696371
83 2.99952219626734
84 2.99956997664061
85 2.99961297897655
86 2.99965168107889
87 2.999686512971
88 2.9997178616739
89 2.99974607550651
90 2.99977146795586
91 2.99979432116027
92 2.99981488904425
93 2.99983340013982
94 2.99985006012584
95 2.99986505411326
96 2.99987854870193
97 2.99989069383174
98 2.99990162444856
99 2.99991146200371
100 2.99992031580334
};
\addlegendentry{Causal};
\addplot [green!50!black, thick]
table {%
0 0
1 0
2 0
3 0
4 0
5 0
6 0
7 0
8 0
9 0
10 0
11 0
12 0
13 0
14 0
15 0
16 0
17 0
18 0
19 0
20 0
21 0
22 0
23 0
24 0
25 0
26 0
27 0
28 0
29 0
30 0
31 0
32 0
33 0
34 0
35 0
36 0
37 0
38 0
39 0
40 0
41 0
42 0
43 0
44 0
45 0
46 0
47 0
48 0
49 0
50 0
51 0
52 0
53 0
54 0
55 0
56 0
57 0
58 0
59 0
60 0
61 0
62 0
63 0
64 0
65 0
66 0
67 0
68 0
69 0
70 0
71 0
72 0
73 0
74 0
75 0
76 0
77 0
78 0
79 0
80 0
81 0
82 0
83 0
84 0
85 0
86 0
87 0
88 0
89 0
90 0
91 0
92 0
93 0
94 0
95 0
96 0
97 0
98 0
99 0
100 0
};
\addlegendentry{SINDy};
\end{axis}

\end{tikzpicture}
\caption{Comparison of prediction capabilities. \capt{In blue predictions of the true model obtained after causal identification, in green the model obtained from SINDy with parameter choices from \cite{brunton2016discovering}. We do not compare to PCMCI in the numerical experiment since PCMCI only reveals which causal influences exist, but not how strong they are, \ie different from our algorithm and SINDy, it does not return a full system model that can be used in simulations.}}
\label{fig:synt_comp}
\end{figure}

Only when lowering the threshold parameter can SINDy recover the true causal structure of the system.
This stresses the general shortcoming of sparse identification methods when identifying the causal structure of a control system.
Depending on the parameter settings, they may or may not recover the true causal structure.
However, their general purpose is different.
They seek a trade-off between model complexity and accuracy.
Thus, neglecting a causal link with a comparably weak influence might be the desired outcome.
In contrast, we seek the true causal structure, independent of how strong the influence is.

Second, we compare our algorithm to the PCMCI\footnote{an adaptation of the PC algorithm, named after its creators Peter Spirtes and Clark Glymour \citep{spirtes2000causation}, which uses a momentary conditional independence (MCI) test to deal with auto-correlations in time-series data} algorithm proposed in \cite{runge2019detecting}\footnote{For our simulations, we use the implementation and parameter settings provided in \url{https://github.com/jakobrunge/tigramite}}.
This algorithm focuses on detecting causal influences from time-series data.
That is, it does not yield an identified system matrix but discovers which variables have a causal influence on which others.
When again exciting the linear system with a chirp signal and running the PCMCI algorithm, we receive an output that we can interpret as
\begin{align}
\label{eqn:res_pcmci}
\begin{split}
x_1(k+1) &= a_1x_1(k) + a_2x_2(k) + a_3x_3(k) + b_1u_1(k)\\
x_2(k+1) &= a_4x_1(k-1) + a_5x_2(k) + a_5x_3(k) + a_6x_3(k-2) + b_2u_2(k) + b_3u_3(k-1)\\
x_3(k+1) &= a_7x_1(k-4) + a_8x_3(k) + b_4u_1(k) + b_5u_3(k),
\end{split}
\end{align}
where all weights $a_i$ and $b_i$ are non-zero.
Besides, the algorithm discovers an influence of $x_1$ on $u_3$.
However, since such links were ruled out by design for the other algorithms, we neglect this here.
Nevertheless, also the equations as stated above are not in line with the actual matrices in \eqref{eqn:syn_example}.
For instance, following \eqref{eqn:res_pcmci}, the variable $x_1$ has a causal influence on $x_2$ and $x_3$ through the non-zero factors $a_4$ and $a_7$.
However, given the upper-triangular structure of the $A$ matrix, $x_3$ is not influenced by any other state variable apart from itself.
Similarly, also $a_4$ should be zero.
Thus, the algorithm does not reveal the true structure of the system.

\section{Conclusion}
\label{sec:conclusion}

We presented a method that identifies the causal structure of dynamical control systems by conducting experiments and analyzing the generated data with MMD-based techniques.
It differs from prior approaches to causal inference in that it uses a controllability notion that is suitable to design experiments for control systems.
We evaluated the method on a real-world robotic system and a simulated quadruple tank system. 
Our algorithm successfully identified the underlying causal structure of both systems, which in turn allowed us to learn a model that accurately generalizes to previously unseen states while reducing computational complexity.

\subsubsection*{Acknowledgments}
The authors would like to thank Manuel W\"uthrich and Alonso Marco Valle for insightful discussions, Steve Heim for valuable feedback, and Vincent Berenz for support with the robot experiments.
This work has received funding from the German Research
Foundation within the SPP 1914 (grant TR 1433/1-1), the Federal Ministry of Education and Research (BMBF) and
the Ministry of Culture and Science of the German State of North Rhine-Westphalia
(MKW) under the Excellence Strategy of the Federal Government and the Länder, the Cyber
Valley Initiative, the Max Planck Society, the Knut and Alice Wallenberg Foundation, and the Swedish Research Council.

\bibliography{tmlr}

\begin{thebibliography}{70}
\providecommand{\natexlab}[1]{#1}
\providecommand{\url}[1]{\texttt{#1}}
\expandafter\ifx\csname urlstyle\endcsname\relax
  \providecommand{\doi}[1]{doi: #1}\else
  \providecommand{\doi}{doi: \begingroup \urlstyle{rm}\Url}\fi

\bibitem[Akaike(1973)]{akaike1998information}
Hirotogu Akaike.
\newblock Information theory and an extension of the maximum likelihood
  principle.
\newblock In \emph{International Symposium on Information Theory}, pp.\
  267--281, 1973.

\bibitem[Anderson \& Moore(2007)Anderson and Moore]{anderson2007optimal}
Brian~DO Anderson and John~B Moore.
\newblock \emph{Optimal Control: Linear Quadratic Methods}.
\newblock Courier Corporation, 2007.

\bibitem[Bashirov \& Kerimov(1997)Bashirov and
  Kerimov]{bashirov1997controllability}
Agamirza~E Bashirov and Kerim~R Kerimov.
\newblock On controllability conception for stochastic systems.
\newblock \emph{SIAM Journal on Control and Optimization}, 35\penalty0
  (2):\penalty0 384--398, 1997.

\bibitem[Berenz et~al.(2021)Berenz, Naveau, Widmaier, W{\"u}thrich, Passy,
  Guist, and B{\"u}chler]{o80}
Vincent Berenz, Maximilien Naveau, Felix Widmaier, Manuel W{\"u}thrich,
  Jean-Claude Passy, Simon Guist, and Dieter B{\"u}chler.
\newblock The o80 {C++} templated toolbox: Designing customized {P}ython {APIs}
  for synchronizing realtime processes.
\newblock \emph{Journal of Open Source Software}, 6\penalty0 (66):\penalty0
  2752, 2021.

\bibitem[Boccaletti et~al.(2006)Boccaletti, Latora, Moreno, Chavez, and
  Hwang]{boccaletti2006complex}
Stefano Boccaletti, Vito Latora, Yamir Moreno, Martin Chavez, and Dong-Uk
  Hwang.
\newblock Complex networks: Structure and dynamics.
\newblock \emph{Physics Reports}, 424\penalty0 (4-5):\penalty0 175--308, 2006.

\bibitem[Bongard \& Lipson(2007)Bongard and Lipson]{bongard2007automated}
Josh Bongard and Hod Lipson.
\newblock Automated reverse engineering of nonlinear dynamical systems.
\newblock \emph{Proceedings of the National Academy of Sciences}, 104\penalty0
  (24):\penalty0 9943--9948, 2007.

\bibitem[Brunton et~al.(2016)Brunton, Proctor, and
  Kutz]{brunton2016discovering}
Steven~L Brunton, Joshua~L Proctor, and J~Nathan Kutz.
\newblock Discovering governing equations from data by sparse identification of
  nonlinear dynamical systems.
\newblock \emph{Proceedings of the National Academy of Sciences}, 113\penalty0
  (15):\penalty0 3932--3937, 2016.

\bibitem[Chebyshev(1867)]{chebyshev1874sur}
Pafnutii~Lvovich Chebyshev.
\newblock Des valeurs moyennes.
\newblock \emph{Journal de Math{\'e}matiques Pures et Appliqu{\'e}es},
  2\penalty0 (12):\penalty0 177--184, 1867.

\bibitem[Chen et~al.(2014)Chen, Zhang, Chan, and Sch{\"o}lkopf]{chen2014causal}
Zhitang Chen, Kun Zhang, Laiwan Chan, and Bernhard Sch{\"o}lkopf.
\newblock Causal discovery via reproducing kernel {H}ilbert space embeddings.
\newblock \emph{Neural Computation}, 26\penalty0 (7):\penalty0 1484--1517,
  2014.

\bibitem[Chwialkowski \& Gretton(2014)Chwialkowski and
  Gretton]{chwialkowski2014kernel}
Kacper Chwialkowski and Arthur Gretton.
\newblock A kernel independence test for random processes.
\newblock In \emph{International Conference on Machine Learning}, pp.\
  1422--1430, 2014.

\bibitem[Chwialkowski et~al.(2014)Chwialkowski, Sejdinovic, and
  Gretton]{chwialkowski2014wild}
Kacper Chwialkowski, Dino Sejdinovic, and Arthur Gretton.
\newblock A wild bootstrap for degenerate kernel tests.
\newblock In \emph{Advances in Neural Information Processing Systems}, pp.\
  3608--3616, 2014.

\bibitem[de~Silva et~al.(2020)de~Silva, Champion, Quade, Loiseau, Kutz, and
  Brunton]{desilva2020pysindy}
Brian de~Silva, Kathleen Champion, Markus Quade, Jean-Christophe Loiseau,
  J.~Kutz, and Steven Brunton.
\newblock {PySINDy}: A {P}ython package for the sparse identification of
  nonlinear dynamical systems from data.
\newblock \emph{Journal of Open Source Software}, 5\penalty0 (49):\penalty0
  2104, 2020.

\bibitem[Demiralp \& Hoover(2003)Demiralp and Hoover]{demiralp2003searching}
Selva Demiralp and Kevin~D Hoover.
\newblock Searching for the causal structure of a vector autoregression.
\newblock \emph{Oxford Bulletin of Economics and Statistics}, 65:\penalty0
  745--767, 2003.

\bibitem[Eichler(2010)]{eichler2010graphical}
Michael Eichler.
\newblock Graphical {G}aussian modelling of multivariate time series with
  latent variables.
\newblock In \emph{International Conference on Artificial Intelligence and
  Statistics}, pp.\  193--200, 2010.

\bibitem[Eichler(2012)]{eichler2012causal}
Michael Eichler.
\newblock Causal inference in time series analysis.
\newblock In Carlo Berzuini, Philip Dawid, and Luisa Bernardinell (eds.),
  \emph{Causality: Statistical Perspectives and Applications}, chapter~22, pp.\
   327--354. John Wiley \& Sons, Ltd, 2012.

\bibitem[Emami-Naeini \& Franklin(1982)Emami-Naeini and
  Franklin]{emami1982deadbeat}
Abbas Emami-Naeini and G~Franklin.
\newblock Deadbeat control and tracking of discrete-time systems.
\newblock \emph{IEEE Transactions on Automatic Control}, 27\penalty0
  (1):\penalty0 176--181, 1982.

\bibitem[Entner \& Hoyer(2010)Entner and Hoyer]{entner2010causal}
Doris Entner and Patrik~O Hoyer.
\newblock On causal discovery from time series data using {FCI}.
\newblock \emph{Probabilistic Graphical Models}, pp.\  121--128, 2010.

\bibitem[Friston et~al.(2003)Friston, Harrison, and Penny]{friston2003dynamic}
Karl~J Friston, Lee Harrison, and Will Penny.
\newblock Dynamic causal modelling.
\newblock \emph{Neuroimage}, 19\penalty0 (4):\penalty0 1273--1302, 2003.

\bibitem[Fukumizu et~al.(2008)Fukumizu, Gretton, Sun, and
  Sch{\"o}lkopf]{fukumizu2008kernel}
Kenji Fukumizu, Arthur Gretton, Xiaohai Sun, and Bernhard Sch{\"o}lkopf.
\newblock Kernel measures of conditional dependence.
\newblock In \emph{Advances in Neural Information Processing Systems}, pp.\
  489--496, 2008.

\bibitem[Gretton et~al.(2012)Gretton, Borgwardt, Rasch, Sch{\"o}lkopf, and
  Smola]{gretton2012kernel}
Arthur Gretton, Karsten~M Borgwardt, Malte~J Rasch, Bernhard Sch{\"o}lkopf, and
  Alexander Smola.
\newblock A kernel two-sample test.
\newblock \emph{Journal of Machine Learning Research}, 13\penalty0
  (Mar):\penalty0 723--773, 2012.

\bibitem[Hespanha(2018)]{hespanha2018linear}
Joao~P Hespanha.
\newblock \emph{Linear Systems Theory}.
\newblock Princeton University Press, 2018.

\bibitem[Johansson(2000)]{johansson2000quadruple}
Karl~Henrik Johansson.
\newblock The quadruple-tank process: A multivariable laboratory process with
  an adjustable zero.
\newblock \emph{IEEE Transactions on Control Systems Technology}, 8\penalty0
  (3):\penalty0 456--465, 2000.

\bibitem[Kalman(1960{\natexlab{a}})]{kalman1960contributions}
Rudolf~Emil Kalman.
\newblock Contributions to the theory of optimal control.
\newblock \emph{Boletin de la Sociedad Matematica Mexicana}, 5\penalty0
  (2):\penalty0 102--119, 1960{\natexlab{a}}.

\bibitem[Kalman(1960{\natexlab{b}})]{kalman1960general}
Rudolf~Emil Kalman.
\newblock On the general theory of control systems.
\newblock In \emph{International Conference on Automatic Control}, pp.\
  481--492, 1960{\natexlab{b}}.

\bibitem[Kappler et~al.(2018)Kappler, Meier, Issac, Mainprice, Cifuentes,
  W{\"u}thrich, Berenz, Schaal, Ratliff, and Bohg]{kappler2018real}
Daniel Kappler, Franziska Meier, Jan Issac, Jim Mainprice, Cristina~Garcia
  Cifuentes, Manuel W{\"u}thrich, Vincent Berenz, Stefan Schaal, Nathan
  Ratliff, and Jeannette Bohg.
\newblock Real-time perception meets reactive motion generation.
\newblock \emph{IEEE Robotics and Automation Letters}, 3\penalty0 (3):\penalty0
  1864--1871, 2018.

\bibitem[Kullback \& Leibler(1951)Kullback and
  Leibler]{kullback1951information}
Solomon Kullback and Richard~A Leibler.
\newblock On information and sufficiency.
\newblock \emph{The Annals of Mathematical Statistics}, 22\penalty0
  (1):\penalty0 79--86, 1951.

\bibitem[Lauer \& Bloch(2018)Lauer and Bloch]{lauer2018hybrid}
Fabien Lauer and G{\'e}rard Bloch.
\newblock \emph{Hybrid System Identification: Theory and Algorithms for
  Learning Switching Models}.
\newblock Springer, 2018.

\bibitem[Liu et~al.(2009)Liu, Lu, L{\"u}, and Hill]{liu2009structure}
Hui Liu, Jun-An Lu, Jinhu L{\"u}, and David~J Hill.
\newblock Structure identification of uncertain general complex dynamical
  networks with time delay.
\newblock \emph{Automatica}, 45\penalty0 (8):\penalty0 1799--1807, 2009.

\bibitem[Ljung(1999)]{ljung1999system}
Lennart Ljung.
\newblock \emph{System Identification: Theory for the User}.
\newblock Prentice Hall PTR, 1999.

\bibitem[Lopez-Paz et~al.(2015)Lopez-Paz, Muandet, Sch{\"o}lkopf, and
  Tolstikhin]{lopez2015towards}
David Lopez-Paz, Krikamol Muandet, Bernhard Sch{\"o}lkopf, and Ilya Tolstikhin.
\newblock Towards a learning theory of cause-effect inference.
\newblock In \emph{International Conference on Machine Learning}, pp.\
  1452--1461, 2015.

\bibitem[Malinsky \& Spirtes(2018)Malinsky and Spirtes]{malinsky2018causal}
Daniel Malinsky and Peter Spirtes.
\newblock Causal structure learning from multivariate time series in settings
  with unmeasured confounding.
\newblock In \emph{ACM SIGKDD Workshop on Causal Discovery}, pp.\  23--47,
  2018.

\bibitem[{Materassi} \& {Innocenti}(2010){Materassi} and
  {Innocenti}]{materassi2010topological}
Donatello {Materassi} and Giacomo {Innocenti}.
\newblock Topological identification in networks of dynamical systems.
\newblock \emph{IEEE Transactions on Automatic Control}, 55\penalty0
  (8):\penalty0 1860--1871, 2010.

\bibitem[Moneta et~al.(2011)Moneta, Chla{\ss}, Entner, and
  Hoyer]{moneta2011causal}
Alessio Moneta, Nadine Chla{\ss}, Doris Entner, and Patrik Hoyer.
\newblock Causal search in structural vector autoregressive models.
\newblock In \emph{NIPS Mini-Symposium on Causality in Time Series}, pp.\
  95--114, 2011.

\bibitem[Mooij et~al.(2013)Mooij, Janzing, and Sch\"{o}lkopf]{moij2013from}
Joris~M. Mooij, Dominik Janzing, and Bernhard Sch\"{o}lkopf.
\newblock From ordinary differential equations to structural causal models: The
  deterministic case.
\newblock In \emph{Conference on Uncertainty in Artificial Intelligence}, pp.\
  440–448, 2013.

\bibitem[Nguyen-Tuong \& Peters(2011)Nguyen-Tuong and Peters]{nguyen2011model}
Duy Nguyen-Tuong and Jan Peters.
\newblock Model learning for robot control: A survey.
\newblock \emph{Cognitive Processing}, 12\penalty0 (4):\penalty0 319--340,
  2011.

\bibitem[Pannocchia et~al.(2005)Pannocchia, Laachi, and
  Rawlings]{pannocchia2005candidate}
Gabriele Pannocchia, Nabil Laachi, and James~B Rawlings.
\newblock A candidate to replace {PID} control: {SISO}-constrained {LQ}
  control.
\newblock \emph{AIChE Journal}, 51\penalty0 (4):\penalty0 1178--1189, 2005.

\bibitem[Pearl(1995)]{pearl1995causal}
Judea Pearl.
\newblock Causal diagrams for empirical research.
\newblock \emph{Biometrika}, 82\penalty0 (4):\penalty0 669--688, 1995.

\bibitem[Pearl(2018)]{pearl2018theoretical}
Judea Pearl.
\newblock Theoretical impediments to machine learning with seven sparks from
  the causal revolution.
\newblock \emph{arXiv preprint arXiv:1801.04016}, 2018.

\bibitem[Pearl \& Mackenzie(2018)Pearl and Mackenzie]{pearl2018book}
Judea Pearl and Dana Mackenzie.
\newblock \emph{The Book of Why: The New Science of Cause and Effect}.
\newblock Basic Books, 2018.

\bibitem[Peters et~al.(2013)Peters, Janzing, and
  Sch{\"o}lkopf]{peters2013causal}
Jonas Peters, Dominik Janzing, and Bernhard Sch{\"o}lkopf.
\newblock Causal inference on time series using restricted structural equation
  models.
\newblock In \emph{Advances in Neural Information Processing Systems}, pp.\
  154--162, 2013.

\bibitem[Peters et~al.(2017)Peters, Janzing, and
  Sch{\"o}lkopf]{peters2017elements}
Jonas Peters, Dominik Janzing, and Bernhard Sch{\"o}lkopf.
\newblock \emph{Elements of Causal Inference: Foundations and Learning
  Algorithms}.
\newblock MIT Press, 2017.

\bibitem[Peters et~al.(2022)Peters, Bauer, and Pfister]{peters2022causal}
Jonas Peters, Stefan Bauer, and Niklas Pfister.
\newblock Causal models for dynamical systems.
\newblock In Hector Geffner, Rina Dechter, and Joseph~Y. Halpern (eds.),
  \emph{Probabilistic and Causal Inference: The Works of Judea Pearl}, pp.\
  671--690. 2022.

\bibitem[Pfister et~al.(2019)Pfister, Bauer, and Peters]{pfister2019learning}
Niklas Pfister, Stefan Bauer, and Jonas Peters.
\newblock Learning stable and predictive structures in kinetic systems.
\newblock \emph{Proceedings of the National Academy of Sciences}, 116\penalty0
  (51):\penalty0 25405--25411, 2019.

\bibitem[Quinn et~al.(2011)Quinn, Coleman, Kiyavash, and
  Hatsopoulos]{quinn2011estimating}
Christopher~J Quinn, Todd~P Coleman, Negar Kiyavash, and Nicholas~G
  Hatsopoulos.
\newblock Estimating the directed information to infer causal relationships in
  ensemble neural spike train recordings.
\newblock \emph{Journal of Computational Neuroscience}, 30\penalty0
  (1):\penalty0 17--44, 2011.

\bibitem[Roll et~al.(2004)Roll, Bemporad, and Ljung]{roll2004identification}
Jacob Roll, Alberto Bemporad, and Lennart Ljung.
\newblock Identification of piecewise affine systems via mixed-integer
  programming.
\newblock \emph{Automatica}, 40\penalty0 (1):\penalty0 37--50, 2004.

\bibitem[Rubenstein et~al.(2018)Rubenstein, Bongers, Sch{\"o}lkopf, and
  Mooij]{rubenstein2018from}
Paul~K. Rubenstein, Stephan Bongers, Bernhard Sch{\"o}lkopf, and Joris~M.
  Mooij.
\newblock From deterministic {ODEs} to dynamic structural causal models.
\newblock In \emph{Conference on Uncertainty in Artificial Intelligence}, 2018.

\bibitem[Rudin(2019)]{rudin2019stop}
Cynthia Rudin.
\newblock Stop explaining black box machine learning models for high stakes
  decisions and use interpretable models instead.
\newblock \emph{Nature Machine Intelligence}, 1\penalty0 (5):\penalty0
  206--215, 2019.

\bibitem[Runge et~al.(2019{\natexlab{a}})Runge, Bathiany, Bollt, Camps-Valls,
  Coumou, Deyle, Glymour, Kretschmer, Mahecha, Mu{\~n}oz-Mar{\'\i}, Nes,
  Peters, Quax, Reichstein, Scheffer, Sch\"{o}lkopf, Spirtes, Sugihara, Sun,
  and Zscheischler]{runge2019inferring}
Jakob Runge, Sebastian Bathiany, Erik Bollt, Gustau Camps-Valls, Dim Coumou,
  Ethan Deyle, Clark Glymour, Marlene Kretschmer, Miguel~D Mahecha, Jordi
  Mu{\~n}oz-Mar{\'\i}, Egbert Nes, Jonas Peters, Rick Quax, Markus Reichstein,
  Marten Scheffer, Bernhard Sch\"{o}lkopf, Peter Spirtes, George Sugihara, Jie
  Sun, and Jako Zscheischler.
\newblock Inferring causation from time series in {E}arth system sciences.
\newblock \emph{Nature Communications}, 10\penalty0 (1):\penalty0 1--13,
  2019{\natexlab{a}}.

\bibitem[Runge et~al.(2019{\natexlab{b}})Runge, Nowack, Kretschmer, Flaxman,
  and Sejdinovic]{runge2019detecting}
Jakob Runge, Peer Nowack, Marlene Kretschmer, Seth Flaxman, and Dino
  Sejdinovic.
\newblock Detecting and quantifying causal associations in large nonlinear time
  series datasets.
\newblock \emph{Science Advances}, 5\penalty0 (11):\penalty0 eaau4996,
  2019{\natexlab{b}}.

\bibitem[Salvi et~al.(2021)Salvi, Lemercier, Liu, Horvath, Damoulas, and
  Lyons]{salvi2021higher}
Cristopher Salvi, Maud Lemercier, Chong Liu, Blanka Horvath, Theodoros
  Damoulas, and Terry Lyons.
\newblock Higher order kernel mean embeddings to capture filtrations of
  stochastic processes.
\newblock \emph{Advances in Neural Information Processing Systems}, 2021.

\bibitem[Schmidt \& Lipson(2009)Schmidt and Lipson]{schmidt2009distilling}
Michael Schmidt and Hod Lipson.
\newblock Distilling free-form natural laws from experimental data.
\newblock \emph{Science}, 324\penalty0 (5923):\penalty0 81--85, 2009.

\bibitem[Sch{\"o}lkopf(2022)]{scholkopf2022causality}
Bernhard Sch{\"o}lkopf.
\newblock Causality for machine learning.
\newblock In Hector Geffner, Rina Dechter, and Joseph~Y. Halpern (eds.),
  \emph{Probabilistic and Causal Inference: The Works of Judea Pearl}, pp.\
  765--804. 2022.

\bibitem[Schoukens \& Ljung(2019)Schoukens and Ljung]{schoukens2019nonlinear}
Johan Schoukens and Lennart Ljung.
\newblock Nonlinear system identification: A user-oriented road map.
\newblock \emph{IEEE Control Systems Magazine}, 39\penalty0 (6):\penalty0
  28--99, 2019.

\bibitem[Schwarz(1978)]{schwarz1978estimating}
Gideon Schwarz.
\newblock Estimating the dimension of a model.
\newblock \emph{The Annals of Statistics}, 6\penalty0 (2):\penalty0 461--464,
  1978.

\bibitem[Shahrampour \& Preciado(2014)Shahrampour and
  Preciado]{shahrampour2014topology}
Shahin Shahrampour and Victor~M Preciado.
\newblock Topology identification of directed dynamical networks via power
  spectral analysis.
\newblock \emph{IEEE Transactions on Automatic Control}, 60\penalty0
  (8):\penalty0 2260--2265, 2014.

\bibitem[Shanmugam et~al.(2015)Shanmugam, Kocaoglu, Dimakis, and
  Vishwanath]{shanmugam2015learning}
Karthikeyan Shanmugam, Murat Kocaoglu, Alexandros~G Dimakis, and Sriram
  Vishwanath.
\newblock Learning causal graphs with small interventions.
\newblock In \emph{Advances in Neural Information Processing Systems}, pp.\
  3195--3203, 2015.

\bibitem[Simchowitz et~al.(2018)Simchowitz, Mania, Tu, Jordan, and
  Recht]{simchowitz2018learning}
Max Simchowitz, Horia Mania, Stephen Tu, Michael~I. Jordan, and Benjamin Recht.
\newblock Learning without mixing: Towards a sharp analysis of linear system
  identification.
\newblock In \emph{Conference on Learning Theory}, pp.\  439--473, 2018.

\bibitem[Sokol \& Hansen(2014)Sokol and Hansen]{sokol2014causal}
Alexander Sokol and Niels~Richard Hansen.
\newblock Causal interpretation of stochastic differential equations.
\newblock \emph{Electronic Journal of Probability}, 19\penalty0 (100):\penalty0
  1--24, 2014.

\bibitem[Solowjow et~al.(2020)Solowjow, Baumann, Fiedler, Jocham, Seel, and
  Trimpe]{solowjow2020kernel}
Friedrich Solowjow, Dominik Baumann, Christian Fiedler, Andreas Jocham, Thomas
  Seel, and Sebastian Trimpe.
\newblock A kernel two-sample test for dynamical systems.
\newblock \emph{arXiv preprint arXiv:2004.11098}, 2020.

\bibitem[Spirtes et~al.(2000)Spirtes, Glymour, Scheines, Heckerman, Meek,
  Cooper, and Richardson]{spirtes2000causation}
Peter Spirtes, Clark~N Glymour, Richard Scheines, David Heckerman, Christopher
  Meek, Gregory Cooper, and Thomas Richardson.
\newblock \emph{Causation, Prediction, and Search}.
\newblock MIT Press, 2000.

\bibitem[Sriperumbudur et~al.(2010)Sriperumbudur, Gretton, Fukumizu,
  Sch{\"o}lkopf, and Lanckriet]{sriperumbudur2010hilbert}
Bharath~K Sriperumbudur, Arthur Gretton, Kenji Fukumizu, Bernhard
  Sch{\"o}lkopf, and Gert~RG Lanckriet.
\newblock Hilbert space embeddings and metrics on probability measures.
\newblock \emph{Journal of Machine Learning Research}, 11\penalty0
  (Apr):\penalty0 1517--1561, 2010.

\bibitem[Sriperumbudur et~al.(2011)Sriperumbudur, Fukumizu, and
  Lanckriet]{sriperumbudur2011universality}
Bharath~K Sriperumbudur, Kenji Fukumizu, and Gert~RG Lanckriet.
\newblock Universality, characteristic kernels and {RKHS} embedding of
  measures.
\newblock \emph{Journal of Machine Learning Research}, 12\penalty0 (7), 2011.

\bibitem[Sriperumbudur et~al.(2012)Sriperumbudur, Fukumizu, Gretton,
  Sch{\"o}lkopf, and Lanckriet]{sriperumbudur2012empirical}
Bharath~K Sriperumbudur, Kenji Fukumizu, Arthur Gretton, Bernhard
  Sch{\"o}lkopf, and Gert~RG Lanckriet.
\newblock On the empirical estimation of integral probability metrics.
\newblock \emph{Electronic Journal of Statistics}, 6:\penalty0 1550--1599,
  2012.

\bibitem[Stephan et~al.(2010)Stephan, Penny, Moran, den Ouden, Daunizeau, and
  Friston]{stephan2010ten}
Klaas~Enno Stephan, Will~D Penny, Rosalyn~J Moran, Hanneke~EM den Ouden, Jean
  Daunizeau, and Karl~J Friston.
\newblock Ten simple rules for dynamic causal modeling.
\newblock \emph{Neuroimage}, 49\penalty0 (4):\penalty0 3099--3109, 2010.

\bibitem[Sunahara et~al.(1974)Sunahara, Kabeuchi, Asada, Aihara, and
  Kishino]{sunahara1974stochastic}
Yoshifumi Sunahara, Teruo Kabeuchi, Yoshiharu Asada, Shin~Ichi Aihara, and
  Kiyotaka Kishino.
\newblock On stochastic controllability for nonlinear systems.
\newblock \emph{IEEE Transactions on Automatic Control}, 19\penalty0
  (1):\penalty0 49--54, 1974.

\bibitem[Szab{\'o} \& Sriperumbudur(2018)Szab{\'o} and
  Sriperumbudur]{szabo2018characteristic}
Zolt{\'a}n Szab{\'o} and Bharath Sriperumbudur.
\newblock Characteristic and universal tensor product kernels.
\newblock \emph{Journal of Machine Learning Research}, 18:\penalty0 233, 2018.

\bibitem[van~den Hof et~al.(2013)van~den Hof, Dankers, Heuberger, and
  Bombois]{van2013identification}
Paul~MJ van~den Hof, Arne Dankers, Peter~SC Heuberger, and Xavier Bombois.
\newblock Identification of dynamic models in complex networks with prediction
  error methods—basic methods for consistent module estimates.
\newblock \emph{Automatica}, 49\penalty0 (10):\penalty0 2994--3006, 2013.

\bibitem[Williams \& Rasmussen(2006)Williams and
  Rasmussen]{williams2006gaussian}
Christopher~KI Williams and Carl~Edward Rasmussen.
\newblock \emph{Gaussian Processes for Machine Learning}, volume~2.
\newblock MIT press Cambridge, MA, 2006.

\bibitem[Yang et~al.(2018)Yang, Katcoff, and Uhler]{yang2018characterizing}
Karren Yang, Abigail Katcoff, and Caroline Uhler.
\newblock Characterizing and learning equivalence classes of causal {DAG}s
  under interventions.
\newblock In \emph{International Conference on Machine Learning}, pp.\
  5541--5550, 2018.

\bibitem[Yu(2010)]{yu2010estimating}
Dongchuan Yu.
\newblock Estimating the topology of complex dynamical networks by steady state
  control: Generality and limitation.
\newblock \emph{Automatica}, 46\penalty0 (12):\penalty0 2035--2040, 2010.

\end{thebibliography}
\bibliographystyle{tmlr}

\appendix
\section{Proof of Theorem~\ref{thm:linear_result}}
\label{sec:linear}
An LTI system with Gaussian noise follows a normal distribution, whose mean and variance are given by
\begin{subequations}
\begin{align}
\label{eqn:lin_sys_mean}
\E[x(t)] &= A^tx(0)+ \sum_{i=0}^{t-1} A^iBu(t-1-i)\\
\label{eqn:lin_sys_var}
\Var[x(t)] &= \sum_{i=0}^{t-1} A^i\Var[v(t-1-i)],
\end{align}
\end{subequations}
where we assume $\Var[v(t)]=\Sigma_\mathrm{v}$ for all $t$.
For such systems, we will first show that if \eqref{eqn:lin_sys_mean} obeys the controllability conditions stated by Kalman, the system is also controllable according to definitions~\ref{def:eps_controllability} and~\ref{def:distr_controllability}. 
\begin{lem}
\label{lem:controllability_linear}
The system in \eqref{eqn:loc_sys_lin} is completely $\epsilon$-controllable in distribution if the deterministic part obeys the controllability condition stated in \cite{kalman1960contributions}.
\end{lem}
\begin{proof}
The expected value in \eqref{eqn:lin_sys_mean} represents the deterministic part of the system.
Thus, according to \cite{kalman1960contributions}, we can design an input trajectory that steers \eqref{eqn:lin_sys_mean} to any point in the state space.
Since we do not assume constraints on the input or the state variables, the desired state can be reached with a trajectory within $q<\infty$ time steps (\cf deadbeat control \citep{kalman1960general,emami1982deadbeat}).
That is, starting from any $x(0)\in\mathcal{X}$ we can steer the system state toward $x(q)\mbeq \xI$ in $q$ time steps and obtain the distribution
\begin{align*}
\E[x(q)] &= \xI\\
\Var[x(q)] &= \sum_{i=0}^{q-1} A^i\Sigma_\mathrm{v}.
\end{align*}
The probability of $\norm{x(q)-\xI}_2^2$ being larger than $\epsilon$ is given by the cumulative distribution function of the normal distribution.
\end{proof}

\fakepar{Proof of theorem~\ref{thm:linear_result}}
We can now prove theorem~\ref{thm:linear_result}.
Since the variance of \eqref{eqn:loc_sys_lin} solely depends on the number of time steps, which is equal for all experiments, distributions can only be different because of their means.
We start with experiments that are designed according to \eqref{eqn:exp_design_state}.
In this case, for distributions to be equal, and, thus, for variables to be non-causal, we need
\begin{align*}
e_i\left(A^t\xI(0) + \sum_{i=0}^{t-1} A^i(Bu(t-1-i)\right)= e_i\left(A^t\xII(0) + \sum_{i=0}^{t-1} A^i(Bu(t-1-i)\right),
\end{align*}
where $e_i\in\R^n$ is the unit vector (\ie a vector with zeros and a single 1 at the $i$th entry).
Since input trajectories are equal, this boils down to 
\begin{align*}
e_iA^t\xI(0) = e_iA^t\xII(0).
\end{align*}
Essentially this means that the component $ij$ of $A^t$ needs to be 0.
This is clearly the case, if there is no influence of $x_i$ on $x_j$, \ie in case variables are non-causal, we have $\mmd=0$.
The event of component $ij$ of $A^t$ being 0 by chance, even though $x_j$ has a causal influence on $x_i$, has probability 0.
Thus, we have that variables are non-causal if $\mmd=0$.

For experiments that are designed according to \eqref{eqn:exp_design_inp}, initial states are equal and, in case variables are non-causal, we have
\begin{align*}
e_i\sum_{i=0}^{t-1} A^i(B\uI(t-1-i)= e_i\sum_{i=0}^{t-1} A^i(B\uII(t-1-i).
\end{align*}
Similar as before, we have equal distributions and, thus, $\mmd=0$ if entries in the $A^iB$ matrices relating $x_i$ and $u_j$ are 0, \ie if there is no causal influence.
The other direction holds since the event of the relevant entries being 0 by chance has probability 0.

\section{Further Example for Assumption~\ref{assume:no_loc_ind}}
\label{sec:example_loc_non_caus}

We provide a further example to empirically validate assumption~\ref{assume:no_loc_ind}.
In practice, regions of local non-causality are often due to hysteresis effects.
For instance, a resting body first needs to overcome static friction.
That is, when the velocity is zero, the input signal needs to overcome a threshold to actually influence the velocity.
In particular, we consider the system
\begin{align}
\label{eqn:hysteresis}
x(k+1) = \begin{cases}
        \begin{pmatrix} 1&0.01\\ 0 & 1\end{pmatrix}x(k) + \begin{pmatrix}0\\0.1\end{pmatrix}u(k) + v(k)&\text{ if }x_1(k)\neq 0 \lor \abs{u(k)} > 0.1\\
        x(k) &\text{ else},
        \end{cases}
\end{align}
where the state is 2-dimensional, the input is scalar, and $v(k)$ is a normally distributed random variable with standard deviation \num{1e-4}.
We excite the system with a ramp function starting at zero and going until two and learn a GP model.
For the GP model, we use a Gaussian kernel for which we optimize the hyperparameters.
We then assume the initial condition of both state variables to be zero and compute the MMD of simulated trajectories for 100 step inputs, with $\uI$ ranging from 0.05 to 5 and $\uII=-\uI$.
The results confirm our findings from example~\ref{exp:loc_ind} since also here assumption~\ref{assume:no_loc_ind} is satisfied (see \figref{fig:hysteresis}).

\begin{figure}
\centering
\begin{tikzpicture}

\definecolor{color0}{rgb}{0.12156862745098,0.466666666666667,0.705882352941177}

\begin{axis}[
ylabel={$\mmdsq$},
xlabel={$u_1$},
label style={font=\scriptsize},
tick label style={font=\scriptsize},
axis lines=left,
grid=both,
width=0.7\textwidth,
height=0.2\textheight,
xmax=4.9
]
\addplot [semithick, color0]
table {%
5 0.834369867900557
2.5 0.893299532006752
1.66666666666667 0.888217513405374
1.25 0.866332807054533
1 0.83855294538461
0.833333333333333 0.808618983295809
0.714285714285714 0.778148145212736
0.625 0.747942009962717
0.555555555555556 0.718429090593826
0.5 0.689845431815465
0.454545454545455 0.66232000057712
0.416666666666667 0.635914540298622
0.384615384615385 0.61065269209442
0.357142857142857 0.586530730572617
0.333333333333333 0.563527676483763
0.3125 0.541612486528391
0.294117647058824 0.520747993441974
0.277777777777778 0.500891116001617
0.263157894736842 0.481997197172686
0.25 0.464021729281291
0.238095238095238 0.446919254140911
0.227272727272727 0.430645494652746
0.217391304347826 0.415157373157702
0.208333333333333 0.400412611920289
0.2 0.386371908352744
0.192307692307692 0.372997326743528
0.185185185185185 0.360251382626186
0.178571428571429 0.348100259640162
0.172413793103448 0.336511340706717
0.166666666666667 0.325453709389854
0.161290322580645 0.314898790470734
0.15625 0.304818030816771
0.151515151515152 0.29518714628638
0.147058823529412 0.285980453449712
0.142857142857143 0.277176183759506
0.138888888888889 0.268751604791819
0.135135135135135 0.260687114089702
0.131578947368421 0.252963782873771
0.128205128205128 0.245564006659005
0.125 0.238470430956194
0.121951219512195 0.23166733065673
0.119047619047619 0.225139382752085
0.116279069767442 0.218874096388943
0.113636363636364 0.212856822313691
0.111111111111111 0.207075728589758
0.108695652173913 0.201518308311557
0.106382978723404 0.19617515907795
0.104166666666667 0.191034556600377
0.102040816326531 0.186087315903241
0.1 0.181323586530173
0.0980392156862745 0.176736292982325
0.0961538461538462 0.172315406318493
0.0943396226415094 0.168053539243733
0.0925925925925926 0.163944356299327
0.0909090909090909 0.159979598124955
0.0892857142857143 0.15615387607463
0.087719298245614 0.152459929935138
0.0862068965517241 0.148892811541498
0.0847457627118644 0.145446546774072
0.0833333333333333 0.142115729431844
0.0819672131147541 0.138895745886491
0.0806451612903226 0.135781493734701
0.0793650793650794 0.132768674775131
0.078125 0.129852768005697
0.0769230769230769 0.127030369124432
0.0757575757575758 0.124297074447291
0.0746268656716418 0.121649640835079
0.0735294117647059 0.119083780706209
0.072463768115942 0.116597657971051
0.0714285714285714 0.114186747889597
0.0704225352112676 0.111848317351806
0.0694444444444444 0.109580240256737
0.0684931506849315 0.107379027778871
0.0675675675675676 0.105242898927744
0.0666666666666667 0.103168516980567
0.0657894736842105 0.101154034575504
0.0649350649350649 0.0991969501134539
0.0641025641025641 0.0972953938903241
0.0632911392405063 0.0954472265632254
0.0625 0.0936504202746931
0.0617283950617284 0.0919036765072881
0.0609756097560976 0.0902038818487784
0.0602409638554217 0.0885506413427195
0.0595238095238095 0.0869417187331461
0.0588235294117647 0.0853757445542865
0.0581395348837209 0.0838510018232615
0.0574712643678161 0.0823665224396561
0.0568181818181818 0.0809204266654305
0.0561797752808989 0.0795120041292742
0.0555555555555556 0.0781393644893048
0.0549450549450549 0.0768013228449005
0.0543478260869565 0.0754973755598096
0.0537634408602151 0.0742257079156117
0.0531914893617021 0.0729857342362645
0.0526315789473684 0.0717763713999282
0.0520833333333333 0.0705962793093592
0.0515463917525773 0.0694445033680788
0.0510204081632653 0.0683209516163937
0.0505050505050505 0.0672239160832915
0.05 0.066152660187133
};
\end{axis}

\end{tikzpicture}
\caption{Simulated MMD for the hysteresis system in \eqref{eqn:hysteresis}. \capt{Similar as for example~\ref{exp:loc_ind}, also here we see that assumption~\ref{assume:no_loc_ind} is satisfied.}}
\label{fig:hysteresis}
\end{figure}

\section{Further Results of the Robot Experiments}
\label{sec:exp_res}

We first provide implementation details for the experiments presented in \secref{sec:eval}.
The initial model estimate is obtained by exciting the system with a chirp signal for \SI{30}{\second} and using the generated data to learn a linear state-space model (\cf \eqref{eqn:loc_sys_lin}) with least squares.
The obtained matrices are
\begin{align*}
A_\mathrm{init} \approx \begin{pmatrix}
0.868 & -0.132 & 0.754 & -0.491 \\
-0.132 & 0.868 & 0.754 & -0.491 \\
-0.132 & -0.132 & 1.754 & -0.491 \\
-0.134 & -0.134 & 0.76 & 0.508
\end{pmatrix}
\quad
B_\mathrm{init} \approx \begin{pmatrix}
0.075 & -0.056 & -0.031 & 0.022 \\
0.074 & -0.055 & -0.031 & 0.022 \\
0.074 & -0.056 & -0.03 & 0.022 \\
0.075 & -0.056 & -0.032 & 0.022
\end{pmatrix}.
\end{align*}
Initial conditions and input trajectories for the causality testing experiments are obtained through sophisticated guesses, as discussed in \secref{sec:framework}.
The found initial conditions and input trajectories yield expected MMDs orders of magnitude above the system's noise level for all joints.
Thus, we need only eight experiments to identify the causal structure.
We design input trajectories of \num{100} time steps for each experiment, repeat the experiment ten times, and use collected data from all experiments for hypothesis testing.
While the squared MMD is always positive, the empirical approximation in \eqref{eqn:mmd_kernel} can become negative since it is an unbiased estimate.
For the test statistic in \eqref{eqn:test_statistic}, we estimate the variance using \num{100} Monte Carlo simulations and obtain the expected value through a noiseless simulation. 
We use $\nu=1$, but as we will see in the results, the empirical MMD is, in all cases, orders of magnitude below or above the threshold. 
Thus, more conservative choices of $\nu$ would yield the same outcome.

In tables~\ref{tab:exp_res_1} and~\ref{tab:exp_res_2}, we present the results of all causality testing experiments conducted on the robotic platform shown in \figref{fig:robot}.
As for the results discussed in \secref{sec:eval}, we always have a clear decision on whether to accept or reject the null hypothesis: The MMD found in experiments is always orders of magnitude larger or smaller than the test statistic.
Also here, we find that all joints can be moved independently of each other and are affected by exactly one input.
When exploiting the revealed causal structure for identifying the system matrices, we obtain
\begin{align*}
A_\mathrm{caus} = \begin{pmatrix}
1 & 0 & 0 & 0 \\
0 & 1 & 0 & 0 \\
0 & 0 & 1 & 0 \\
0 & 0 & 0 & 1
\end{pmatrix}
\quad
B_\mathrm{caus} \approx \begin{pmatrix}
0.013 & 0 & 0 & 0 \\
0 & 0.007 & 0 & 0 \\
0 & 0 & 0.01 & 0 \\
0 & 0 & 0 & 0.01
\end{pmatrix}.
\end{align*}

\begin{table}
\centering
\caption{Results of the causal structure identification for a robot arm. 
\capt{Causal influences of joints on each other.}}
\label{tab:exp_res_1}
\begin{tabular}{ccc}
\toprule
Joint $\rightarrow$ & Experimental & Test \\
Joint & MMD & statistic \\
\midrule
$x_1\rightarrow x_1$ & 0 & \num{1.65e-04} \\
$x_1\rightarrow x_2$ & 0 & \num{1.79e-04} \\
$x_1\rightarrow x_3$ & 0 & \num{2.39e-4} \\
$x_1\rightarrow x_4$ & \num{2.43e-13} & \num{1.61e-4}\\
$x_2\rightarrow x_1$ & 0 & \num{5.6e-7}\\
$x_2\rightarrow x_2$ & \num{-2.8e-18} & \num{4.58e-7}\\
$x_2\rightarrow x_3$ & 0 & \num{3.56e-7}\\
$x_2\rightarrow x_4$ & \num{1.82e-13} & \num{6.91e-7}\\
$x_3\rightarrow x_1$ & 0 & \num{5.81e-7}\\
$x_3\rightarrow x_2$ & 0 & \num{4.54e-7}\\
$x_3\rightarrow x_3$ & \num{-1.38e-18} & \num{6.16e-7}\\
$x_3\rightarrow x_4$ & \num{1.29e-16} & \num{5.2e-7}\\
$x_4\rightarrow x_1$ & 0 & \num{4.99e-7}\\
$x_4\rightarrow x_2$ & 0 & \num{4.66e-7}\\
$x_4\rightarrow x_3$ & \num{9.63e-15} & \num{5.8e-7}\\
$x_4\rightarrow x_4$ & \num{-5.44e-15} & \num{5.8e-7}\\
\bottomrule
\end{tabular}
\end{table}

\begin{table}
\centering
\caption{Results of the causal structure identification for a robot arm. 
\capt{Causal influences of inputs on joints.}}
\label{tab:exp_res_2}
\begin{tabular}{ccc}
\toprule
Input $\rightarrow$ & Experimental & Test\\
Joint & MMD & statistic \\
\midrule
$u_1\rightarrow x_1$ & \num{0.04}& \num{1.18e-5} \\
$u_1\rightarrow x_2$ & 0 & \num{5.38e-7} \\
$u_1\rightarrow x_3$ & 0 & \num{6.51e-7} \\
$u_1\rightarrow x_4$ & 0 & \num{4.47e-7}\\
$u_2\rightarrow x_1$ & 0 & \num{5.09e-5}\\
$u_2\rightarrow x_2$ & \num{0.04} & \num{1.15e-5}\\
$u_2\rightarrow x_3$ & \num{3.51e-14} & \num{5.95e-7}\\
$u_2\rightarrow x_4$ & \num{3.51e-14} & \num{4.67e-7}\\
$u_3\rightarrow x_1$ & \num{2.31e-13} & \num{5.26e-5}\\
$u_3\rightarrow x_2$ & \num{1.4e-13} & \num{4.28e-5}\\
$u_3\rightarrow x_3$ & \num{0.04} & \num{1.18e-5}\\
$u_3\rightarrow x_4$ & \num{2.25e-12} & \num{4.36e-7}\\
$u_4\rightarrow x_1$ & 0 & \num{4.47e-5}\\
$u_4\rightarrow x_2$ & 0 & \num{5.11e-5}\\
$u_4\rightarrow x_3$ & 0 & \num{5.69e-7}\\
$u_4\rightarrow x_4$ & \num{0.04} & \num{6.58e-4}\\
\bottomrule
\end{tabular}
\end{table}

\section{Further Results of the Quadruple Tank Experiments}
\label{sec:4tank_add}

We discretize the quadruple tank system with a time-step of \SI{100}{\milli\second}.
We choose $A_i = \SI{50}{\centi\meter\squared}$ for all tanks, $a_{1,2} = \SI{0.242}{\centi\meter\squared}$, $a_{3,4} = \SI{0.242}{\centi\meter\squared}$, and the valve parameters $\zeta_{1,2} = 0.5$.
For the initial model learning, we excite the system for \num{5000} time steps.
During excitation, the input is sampled from a uniform distribution with an interval $[0,60]$.
We use the generated data to learn a GP with a Gaussian kernel for each state variable.
Having identified a model, we follow \algref{alg:ext_approach} to identify the causal structure.
As for the robot experiments, we repeat each experiment 10 times and use 50 simulations to estimate the standard deviation.
We choose $\nu=10$ in \eqref{eqn:test_statistic} to avoid false positives.
The results are collected in tables~\ref{tab:4tank_res_ext_1} and~\ref{tab:4tank_res_ext_2}.

\begin{table}
\centering
\caption{Results of the causal structure identification for a quadruple tank system. 
\capt{Causal influences of the tanks on each other.}}
\label{tab:4tank_res_ext_1}
\begin{tabular}{ccc}
\toprule
Tank $\rightarrow$ & Experimental & Test \\
Tank & MMD & statistic \\
\midrule
$x_1\rightarrow x_1$ & \num{2.78e-01} & \num{9.81e-04} \\
$x_1\rightarrow x_2$ & \num{7.39e-06} & \num{9.55e-05} \\
$x_1\rightarrow x_3$ & \num{1.47e-05} & \num{9.78e-05} \\
$x_1\rightarrow x_4$ & \num{9.45e-05} & \num{1.01e-03}\\
$x_2\rightarrow x_1$ & \num{3.57e-06} & \num{2.35e-05}\\
$x_2\rightarrow x_2$ & \num{2.74e-01} & \num{4.58e-04}\\
$x_2\rightarrow x_3$ & \num{2.32e-06} & \num{4.31e-06}\\
$x_2\rightarrow x_4$ & \num{3.54e-05} & \num{1.41e-04}\\
$x_3\rightarrow x_1$ & \num{1.62e-01} & \num{1.29e-04}\\
$x_3\rightarrow x_2$ & \num{5.36e-06} & \num{3.28e-04}\\
$x_3\rightarrow x_3$ & \num{3.44e-01} & \num{2.75e-04}\\
$x_3\rightarrow x_4$ & \num{9.59e-06} & \num{3.46e-05}\\
$x_4\rightarrow x_1$ & \num{4.28e-06} & \num{4.04e-05}\\
$x_4\rightarrow x_2$ & \num{1.45e-01} & \num{8.74e-05}\\
$x_4\rightarrow x_3$ & \num{2.35e-06} & \num{6.56e-06}\\
$x_4\rightarrow x_4$ & \num{3.57e-01} & \num{1.42e-04}\\
\bottomrule
\end{tabular}
\end{table}

\begin{table}
\centering
\caption{Results of the causal structure identification for a quadruple tank system. 
\capt{Causal influences of inputs on tanks.}}
\label{tab:4tank_res_ext_2}
\begin{tabular}{ccc}
\toprule
Input $\rightarrow$ & Experimental & Test\\
Tank & MMD & statistic \\
\midrule
$u_1\rightarrow x_1$ & \num{1.8e-01}& \num{1.01e-04} \\
$u_1\rightarrow x_2$ & \num{8.04e-03} & \num{8.74e-05} \\
$u_1\rightarrow x_3$ & \num{2.46e-08} & \num{3.8e-05} \\
$u_1\rightarrow x_4$ & \num{1.56e-01} & \num{8.61e-05}\\
$u_2\rightarrow x_1$ & \num{7.55e-03} & \num{9.86e-05}\\
$u_2\rightarrow x_2$ & \num{1.78e-01} & \num{1.12e-04}\\
$u_2\rightarrow x_3$ & \num{1.58e-01} & \num{3.77e-05}\\
$u_2\rightarrow x_4$ & \num{5.44e-06} & \num{4.46e-05}\\
\bottomrule
\end{tabular}
\end{table}

\end{document}